\documentclass[preprint]{article}
\usepackage{neurips_2026}

\usepackage[utf8]{inputenc} 
\usepackage[T1]{fontenc}
\usepackage{graphicx}
\usepackage{hyperref}
\usepackage{url}
\usepackage{booktabs}
\usepackage{amsfonts,amsmath,amsthm}
\usepackage{nicefrac}
\usepackage{caption}
\usepackage{xcolor}
\usepackage{tikz}
\usepackage{adjustbox}
\usepackage{float}
\usepackage[ruled,vlined]{algorithm2e}

\newtheorem{proposition}{Proposition}

\newtheorem{definition}{Definition}

\title{OPTIMUS-Prime: Minimal \& Sufficient Concept Explanations for Deep Vision Models}

\author{%
  Arthur Hoarau\\
  Université de Lorraine, CentraleSupélec\\
  Loria, CNRS,  Metz, France\\
  \texttt{arthur.hoarau@centralesupelec.fr}\\
  \And
  Chenrui Zhu\\
  Universit\'e de technologie de Compi\`egne\\
  UMR CNRS 7253 Heudiasyc, France\\
  \texttt{chenrui.zhu@utc.fr}\\
  \And
  Vu‑Linh Nguyen\\
  Universit\'e de technologie de Compi\`egne\\
  UMR CNRS 7253 Heudiasyc, France\\
  \texttt{vu-linh.nguyen@utc.fr}\\
}

\begin{document}

\maketitle

\begin{abstract}
The growing demand for transparency in automated decision-making has propelled eXplainable Artificial Intelligence (XAI) to the forefront of machine learning research. In computer vision, however, existing explanation methods often prioritize end-user accessibility at the expense of formal guarantees, leaving a critical gap between practical utility and theoretical rigor.
In this paper, we address this gap by introducing OPTIMUS, a novel framework for generating concept-based visual explanations for deep classification models. OPTIMUS explanations take the form of visual heatmaps that not only remain interpretable to end users, but are grounded in the well-established theory of prime implicants, providing formal guarantees that have been largely absent from existing saliency-based methods.
Specifically, OPTIMUS explanations satisfy two desirable properties: sufficiency, ensuring that the highlighted concepts provably guarantee the classifier's prediction, and minimality, ensuring that no strict subset of those concepts retains this guarantee. Together, these properties yield explanations that are both logically tight and visually coherent.
We validate our approach on a visual classification benchmark, demonstrating that OPTIMUS heatmaps naturally and faithfully surface the decision-relevant concepts underlying model predictions.
\end{abstract}

\section{Introduction}
\begingroup
\renewcommand{\thefootnote}{}
\footnotetext{Link to the code: \url{https://anonymous.4open.science/r/optimus-prime-075D}}
\endgroup

The increasing complexity of modern machine learning architectures has created a fundamental tension between predictive performance and interpretability, driving significant interest in explainability methods. A central challenge in deep vision is finding the right balance between explanations that are theoretically grounded from a learning perspective and those that remain accessible to end users.

A common approach to model interpretability relies on saliency maps, which highlight the input regions most influential to a model's prediction. Several strategies have been proposed in the literature. A first consists in perturbing the inputs and observing the resulting changes in deeper neuron activations~\cite{Zeiler2014, zintgraf2017}. A second strategy, exemplified by Grad-CAM and Guided Grad-CAM~\cite{selvaraju2020}, is specific to convolutional neural networks (CNNs) and leverages the gradient of a target class score, back-propagated to the last convolutional layer, to produce a coarse localization map. In this paper, we focus on a third strategy, in which the gradient of the predicted class score is back-propagated to the input space in a single forward-backward pass~\cite{simonyan2014}. Intuitively, a high saliency value at pixel $i$ indicates that a small perturbation at that location would significantly alter the model's output, suggesting that the model relies on that region to reach its decision.
Such gradient-based methods, however, lack formal guarantees and can exhibit sensitivity to input noise and model parameter choices. To address some of these limitations, several extensions have been proposed, including Integrated Gradients~\cite{sundararajan2017}, which accumulates gradients along a path from a baseline to the input, and DeepLIFT~\cite{shrikumar2017}, which attributes contributions by comparing neuron activations to a reference. Both methods will be considered in this article. Alternatively, concept-based explanation methods such as TCAV~\cite{kim2017} probe the network's internal representations by testing the sensitivity of predictions to user-defined concept directions, offering a higher-level abstraction than pixel-wise saliency.

These techniques belong to a broader family of methods aimed at enhancing the interpretability of machine learning models, collectively known as eXplainable Artificial Intelligence (XAI)~\cite{das2020opportunities}. Saliency heatmaps are \emph{post-hoc}, applied after and independently from the training procedure, and partially \emph{model-agnostic}, in the sense that they do not assume a specific internal architecture. The latter qualification is only partially accurate: several backpropagation-based techniques require specific architectures or become intractable for others, yielding \emph{model-specific} explanations, as opposed to purely \emph{model-agnostic} approaches such as SHAP~\cite{Lundberg17} or LIME~\cite{Lime16}.

In a general classification setting, deep neural networks act as hierarchical feature extractors~\cite{li2015}, building increasingly abstract representations of the input space to facilitate decision-making at deeper layers. In this work, we leverage the notion of \emph{concepts} as the semantic information captured at a given layer~\cite{bau_network_2017, bau2020, karpathy2015}. In the first layer, concepts correspond to raw pixels while in the final layer, they correspond to class labels (e.g., ``dog'' or ``cat'' in an animal classification task). Intermediate layers capture concepts of varying abstraction~\cite{mikolov2013}, potentially encoding semantically meaningful features such as ``pointy ears'' or ``vertical pupils'', though these representations can be highly entangled, subject to superposition~\cite{elhage2022,cheung2019}, and remain difficult to define precisely.

In the formal XAI literature, \emph{prime implicants} provide sufficient and subset-minimal explanations for Boolean classifiers~\cite{JoaoLogicXAI,Ignatiev2020FromCT,ShihCD18}: fixing the variables in the explanation is enough to guarantee the model's prediction, and no proper subset of these variables preserves this guarantee. Transposed to our setting, variables become concepts, and the Boolean function becomes the classifier's decision rule. An OPTIMUS prime implicant explanation thus identifies the subset-minimal set of concepts that provably guarantees the prediction. This explanation can then be backpropagated into the input space, yielding representations that are both human-interpretable and endowed with formal sufficiency and minimality guarantees.

A closely related contribution was recently published~\cite{soria_formal_2026}, in which the authors propose sufficient and minimal explanations for deep vision. Their approach, however, is restricted to prototype-based network architectures~\cite{davoodi_2023}, making it strongly \emph{model-specific}. As a consequence, their prime implicant formulation is inherently tied to the prototype structure and cannot be directly applied to general deep networks. We relax this constraint and extend such explanations to a broader class of deep neural networks, moving toward more \emph{model-agnostic} explanations. Furthermore, their method does not backpropagate explanations into the input space and omits end-user visualizations, relying solely on numerical metrics, limiting the intuitive accessibility of the results. Nevertheless, their work introduces the compelling notion of \emph{abductive latent explanations}, arguing that ``the pixel-level is not the correct level of abstraction for the human final user of the explanation''. We strongly subscribe to this view, and build upon it throughout this paper.

The remainder of this paper is structured as follows. Section~\ref{sec:optimus} details the proposed OPTIMUS framework. Section~\ref{sec:concepts} formalizes the notion of concepts. Section~\ref{sec:prime} extends prime implicants to neural network concepts and introduces a direct search algorithm, while Section~\ref{sec:heatmaps} addresses their backpropagation to the input space. Section~\ref{sec:experiments} validates the framework on a visual Cat-Dog-Bird classification benchmark, demonstrating the effectiveness of OPTIMUS explanations in practice. Section~\ref{sec:discussion} discusses limitations and open challenges, and Section~\ref{sec:conclusion} concludes the article.

\section{OPTIMUS-Prime}\label{sec:optimus}

The OPTIMUS-Prime pipeline is illustrated in Figure~\ref{fig:OPTIMUS}, which operates in two stages. First, given a prediction from a deep vision model, a linear analysis is performed in the latent concept space of an intermediate layer. Leveraging the linearity of the final classification layer, this analysis identifies a prime implicant (PI): a minimal and sufficient subset of concepts, that provably guarantees the predicted class over all rivals. Second, the selected concepts are projected back into the input space via gradient-based attribution methods, namely Integrated Gradients or DeepLIFT, which attribute each pixel's relevance with respect to the selected concepts only. The result is an OPTIMUS heatmap in which highlighted regions directly correspond to the minimal and sufficient concepts underlying the model's decision. This is what gives OPTIMUS its edge over saliency maps: rather than measuring influence, it identify logical sufficiency with a minimality guarantee.

\begin{figure}
    \centering
    \includegraphics[width=\linewidth]{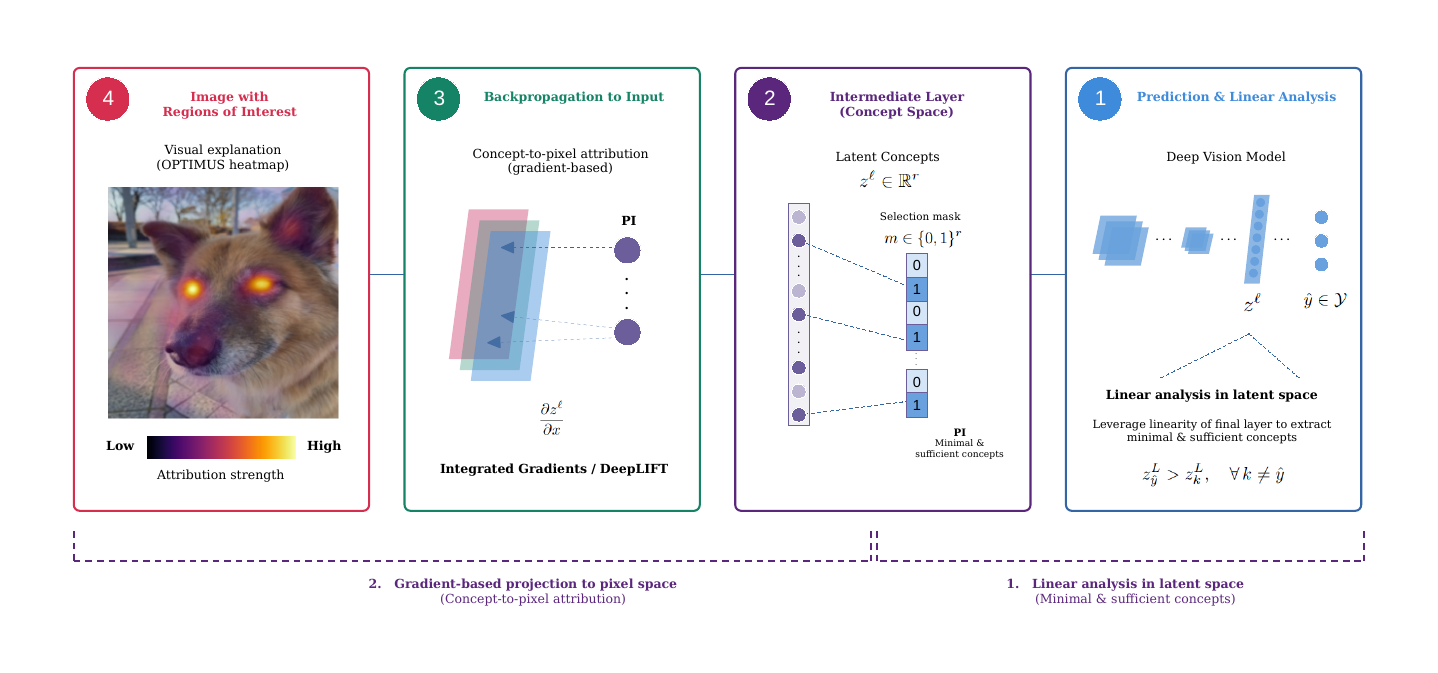}
    \caption{OPTIMUS-Prime: A linear analysis in latent space identifies a prime implicant (PI), a subset of concepts encoded by a selection mask (1-2). These concepts are then backpropagated to the input space via Integrated Gradients or DeepLIFT (3), yielding a visual heatmap where highlighted regions directly correspond to the concepts driving the model's prediction (4).}
    \label{fig:OPTIMUS}
\end{figure}

The following sections formally detail each component of the framework. Section~\ref{sec:concepts} introduces the notion of concepts as the semantic information captured at a given layer. Section~\ref{sec:prime} then defines prime implicants in the context of neural network classifiers and proposes an efficient direct search algorithm. Finally, Section~\ref{sec:heatmaps} addresses the backpropagation of the selected concepts into the input space, bridging the gap between latent guarantees and human-interpretable visualizations.

\subsection{OPTIMUS: Optimal concept explanations}\label{sec:concepts}

We consider a multiclass classification setup with input space \(\mathcal{X} \subseteq \mathbb{R}^P\), where \(P\) denotes the input dimension, and the label space is $\mathcal{Y} := \{1,\dots,K\}$. We denote by \((X,Y)\) a pair of random variables and $(x_n, y_n)$ i.i.d draws from the unknown joint distribution $P_{X,Y}$ forming a dataset $\mathcal{D} := \{(x_n, y_n)\}_{n=1}^N$ of size~$N$. A neural network classifier is a measurable function $h : \mathcal{X} \to \mathcal{Y}$,
which is implemented as a composition of $L$ layers.
For any given input \(x \in \mathcal{X}\), we denote $\hat{y}(x) := h(x) \in \mathcal{Y}$ the predicted class.
We denote by $z^{\ell}(x) \in \mathcal{Z}^{\ell} \subseteq \mathbb{R}^{r^\ell}$ the output of the $\ell$-th hidden layer, for $\ell \in \{1, \dots, L-1\}$. 
We assume that the network can be decomposed as
\begin{equation}
  h(x) = \arg\max_{k \in \{1,\dots,K\}} z^{L}_k(x),    
\end{equation}
where $z^{L} \in \mathcal{Z}^{L} \subseteq \mathbb{R}^{K}$ is the output of the final linear layer, producing one logit $z^{L}_k$ per class.

\begin{definition}[Concept]
    A concept $\circ^\ell \in \{1,..., r^\ell\}$ reflects the semantic information captured in order to activate the $\circ^\ell$-th neuron $z^\ell_{\circ^\ell}$ in layer $\ell$.
\end{definition}
The concepts encoded at the input layer $z^0(x) = x$ correspond to raw pixels, while those at the final layer $z^L$ correspond to class-level representations (e.g., the concept ``cat'' in an animal classification task). One might expect that concepts become increasingly abstract with depth. In practice, however, intermediate-layer concepts rarely align naturally with human-interpretable features~\cite{cheung2019}, such as ``pointy ears'' or ``vertical pupils'', even in tasks where such features are semantically meaningful. Multiple definitions have been proposed to characterize such concepts~\cite{elhage2022, soria_formal_2026}, and we deliberately adopt a minimal one, free of 
any human-interpretability requirements, so as to preserve the generality of our framework. A related line of work constrains the network architecture itself to produce human-interpretable concept predictions as intermediate outputs~\cite{koh_concept_2020}, at the cost of requiring concept-level supervision at training time, a constraint our framework deliberately avoids.

We introduce OPTIMUS, a novel explanation framework that provides users with a heatmap grounded in a minimal and sufficient set of concepts to preserve the model prediction.
\begin{definition}[OPTIMUS explanation]
    A matrix $\mathbf{m} \in \mathbb{R}^{H \times W}$ where each entry $\mathbf{m}_{i,j}$ 
    reflects the relevance of pixel $(i,j)$ w.r.t.\ a set of concepts $\mathcal{O}^{\ell} \subseteq \{1,\dots,r^\ell\}$ in a latent space $\ell \in \{1,\dots,L\}$, is an OPTIMUS explanation if and only if $\mathcal{O}^{\ell}$ is a prime implicant at layer $\ell$.
\end{definition}
This definition leverages the notion of prime implicants, formally introduced in the next section, to ensure both the minimality and sufficiency of the highlighted concepts w.r.t. the model prediction. An OPTIMUS is then a human-interpretable projection of minimal latent concepts in the input space.

\subsection{Multiclass prime implicants}\label{sec:prime}

We now introduce the notion of prime implicants. Informally, a PI is a subset-minimal set of variables sufficient to guarantee the prediction of the model, regardless of the values taken by the remaining variables.

Let $f:\mathcal Z\to\mathbb R^K$ be a linear multiclass classifier function such that 
\begin{equation}
    f_k(x)=\sum_{q=1}^r x_q w_q^{(k)}+b^{(k)},
    \qquad k\in\{1,\dots,K\},
\end{equation}
defined on any multi-dimensional box (in this article, we consider the unit hypercube).
Let consider an instance $x^*\in [0,1]^r$ such that class $c$ is predicted 
\begin{equation}
    f_c(x^*)> f_k(x^*),\qquad \forall k\in\{1,\dots,K\} \setminus c.
\end{equation}
For each rival class $k \in \{1, ..., K\} \setminus c$, let respectively define the dominance weights and biases
\begin{equation}
    \Delta_q^{(k)}:=w_q^{(c)}-w_q^{(k)},\qquad\delta^{(k)}:=b^{(c)}-b^{(k)}.
\end{equation}
The dominance margin between class $c$ and class $k$ is defined as
\begin{equation}
    f_c(x)-f_k(x) =\delta^{(k)}+\sum_{q=1}^r x_q\Delta_q^{(k)}.
\end{equation}
Let $I = \{1,\dots,r\}$ be the set of all feature indices. For a given 
input $x^* \in \mathcal{Z}$ and a subset $S \subseteq I$, we interpret 
$S$ as the set of fixed features, whose values are locked to 
$x^*_q$, while the remaining features $I \setminus S$ are free to vary 
in $[0,1]^r$. The worst-case margin for class $k$ under this partial freedom is defined as
\begin{equation}
    \Gamma_k(S):= \delta^{(k)}+\sum_{q\in S} x_q^*\Delta_q^{(k)}+\sum_{q\in I\setminus S}\min_{v\in[0,1]} v\,\Delta_q^{(k)}.
\end{equation}

\begin{definition}[Multiclass implicant]
A candidate $S\subseteq I$ is a multiclass implicant for class $c$ at $x^*$ iff $\forall x\in [0, 1]^r$ such that $\forall q\in S, x_q = x^*_q$
\begin{equation}
    f_c(x) > f_k(x),
\quad \forall k\in\{1, ..., K\} \setminus c.
\end{equation}
Equivalently, $S$ is a multiclass implicant iff
\begin{equation}
\Gamma_k(S)>0,
\qquad \forall k\in\{1, ..., K\} \setminus c.
\end{equation}
\end{definition}

\begin{definition}[Multiclass prime implicant]
A subset $S\subseteq I$ is a multiclass prime implicant (PI) for class $c$ at $x^*$ iff $S$ is an implicant~\footnote{A singleton implicant is considered prime.
} and
\begin{equation}
    \exists q\in S, \exists k\in\{1, ..., K\} \setminus c,\qquad \Gamma_k(S\setminus{\{q\}}) \leq 0.
\end{equation}
\end{definition}
A multiclass PI is a subset-minimal set of features ensuring dominance over all rival classes simultaneously. Recent extensions of formal explanations to more expressive classifiers~\cite{izza2020} further motivate the need for scalable PI search procedures in the multiclass setting.

\subsubsection{Prime implicant search}

Finding a prime implicant in the linear binary classifier has been studied by~\cite{marques2020explaining}. Their work also considers the search for a cardinality-minimum PI in the binary setting. The objective is to obtain an PI explanation as small as possible while still guaranteeing the prediction. This is particularly relevant in image and visual classification, where large explanations may be difficult to interpret. Inspired by this approach, we extend the search procedure to the multiclass setting. 

Given an instance $x^*$, we propose the following direct search by starting from the complete set of features $I$.
Let $c = \arg\max_k \;f_k(x^*)$ be the predicted class of the classifier.
For each rival class $k\neq c$ let define
\begin{equation}
a_k:=\delta^{(k)}+\sum_{q=1}^r \min_{v\in[0,1]} v\,\Delta_q^{(k)},
\qquad
g_q^{(k)}:=x_q^*\Delta_q^{(k)}-\min_{v\in[0,1]} v\,\Delta_q^{(k)},
\end{equation}
such that the worst-case margin can be rewritten as
\begin{equation}
\Gamma_k(S)=a_k+\sum_{q\in S} g_q^{(k)}.
\end{equation}
A feature $q\in S$ can then be removed if $\Gamma_k(S\setminus\{q\})>0, \forall k\neq c$, which can be simplified for computation as
\begin{equation}
a_k+\sum_{j\in S\setminus\{q\}} g_j^{(k)} >0,
\qquad \forall k\in\{1,\dots,K\}\setminus c.
\end{equation}

\begin{proposition}\label{prop:direct_search_pi}
Assume that $x^*$ is classified as class $c$. Starting from $S=I$, by repeatedly removing any feature $q\in S$ such that \(\forall k\neq c \quad \Gamma_k(S\setminus\{q\})>0\), the resulting set $S$ is a multiclass PI when no further removal can guarantee the constraint.
\end{proposition}
Note, however, that the final PI depends on the order in which removable features are examined. Proof for Proposition~\ref{prop:direct_search_pi} is provided in Appendix~\ref{appendix:direct_search_pi}. The quantities \(a_k\) and \(g_q^{(k)}\) can be precomputed for all \(q \in I\) and all rival classes \(k \in \{1, \cdots, K \} \setminus c\) in \(O(Kr)\) time. Testing whether a feature \(q\in S\) can be removed costs \(O(K)\). Hence, pruning features in a fixed order costs \(O(Kr)\). 

In the binary case, variables are ranked once and for all according to their margin contribution, and then greedily selected in decreasing order of contribution, or, similar to Proposition \ref{prop:direct_search_pi}, removed in ascending order of contribution. In the multiclass case, however, each variable contributes differently to the various pairwise margin, therefore no single static ordering is available. Hence, the search for a cardinality-minimum multiclass PI can be formulated as a combinatorial optimization problem. Let introduce binary variables $s_q\in\{0,1\}$, where $s_q=1$ feature $q$ is selected. The problem becomes
\begin{equation}
\begin{aligned}
    & \min_{s\in\{0,1\}^r} \sum_{q=1}^r s_q, \\
    \text{subject to } \quad & a_k+\sum_{q=1}^r s_q g_q^{(k)} > 0,
    \qquad \forall k\neq c.
\end{aligned}
\end{equation}

Any optimal solution is a multiclass implicant of minimum cardinality. 

Note that an exact solution to find a minimum-cardinality multiclass PI is possible via mixed-integer linear programming by replacing strict inequality 0 by some fixed tolerance \(\varepsilon > 0\). Nevertheless, the problem remains combinatorial in general and becomes impractical for large feature sets. We therefore propose some adaptive greedy rules. One can select among removable features. We thus define the \textsc{LowGain} as
\begin{equation}\label{eq:rule_lowgain}
   q^\star\in\arg\min_{q\in R(S)}\sum_{k\neq c} g_q^{(k)},
\end{equation}
which removes the features with the smallest total contribution to the margins. Elements \(g_q^{(k)}\) can be precomputed in \(O(Kr)\) and sorted once in \(O(r\log r)\). Hence, the total complexity is \(O(Kr+r\log r)\). We also define an alternative \textsc{MaxMin} rule as
\begin{equation}\label{eq:rule_maxmin}
    q^\star\in\arg\max_{q\in R(S)}\min_{k\neq c}\Gamma_k(S\setminus\{q\}),
\end{equation}
which removes the feature that keeps the remaining worst-case margin as large as possible. Recomputing such a greedy choice by scanning all candidates at each step yields a worst-case complexity of \(O(Kr^2)\). The algorithm for finding a multiclass PI is provided in Appendix~\ref{alg:find_multiclass_pi}.

\subsubsection{OPTIMUS Prime implicant}

This section outlines the assumptions and requirements under which the proposed framework maintains its theoretical guarantees.

An OPTIMUS prime implicant is defined at a layer $\ell$ that is imperatively followed by a linear function. This is naturally satisfied at layer $L-1$, which is directly followed by the final linear layer, an architecture adopted by most neural networks. The second requirement is more restrictive: the relative input space must be bounded. This can be enforced by applying a sigmoid activation function at the penultimate layer, constraining the representations to the hypercube $[0,1]^{r^{L-1}}$. However, such a restriction may be too costly in practice, as modern activation functions such as ReLU are positively unbounded. One possible remedy is to replace the sigmoid by an unbounded activation and instead sample the instance space to derive statistical bounds, at the cost of relaxing our exact theoretical guarantees to weaker PAC statistical guarantees.

An OPTIMUS prime implicant $\mathcal{O}^{\ell}$ provides the following guarantee. For an observation $x^*$, $\forall z^{\ell}(x) \in [0,1]^{r^\ell}$ such that $z^{\ell}_{\circ}(x) = z^{\ell}_\circ(x^*)$ for all $ \circ\in\mathcal{O}^{\ell}$
\begin{equation}
    z^L_{\hat{y}(x^*)}(x) > z^L_k(x), \quad \forall\, k \neq \hat{y}(x^*),
\end{equation}
and $\mathcal{O}^{\ell}$ is minimal with respect to this property.

In other words, the concepts $\circ \in \mathcal{O}^{\ell}$ are sufficient to guarantee the prediction $\hat{y}(x^*)$, and minimal in the sense that removing any single concept from $\mathcal{O}^{\ell}$ may alter 
the prediction.

\subsection{Visual gradient backpropagation}\label{sec:heatmaps}

In this section, we focus on the backpropagation of a specific concept's activation into the input space. For this purpose, we use Integrated Gradients~\cite{sundararajan2017} and DeepLIFT~\cite{shrikumar2017}.

\subsubsection{Integrated Gradients}

For this method, we need to compare our input $x$ to a baseline $x'$. In our image setup, the baseline $x'$ simply is a uniformly black image.
Le $\gamma(\sigma)$ be a linear path between $x'$ and $x$ such that 
\begin{equation}
    \gamma(\sigma) = x' + \sigma\times(x-x'), \quad \sigma\in [0,1].
\end{equation}
Given a function $F : \mathcal{X} \rightarrow[0,1]$, Integrated Gradients $\mathrm{IG}$ along the ``$i$-th'' dimension are defined as the path integral of the gradient along the linear path from the baseline to the target
\begin{equation}
    \mathrm{IG}_i(x) =(x_i - x_i')\times\int_{\sigma=0}^1\frac{\partial F(\gamma(\sigma))}{\partial x_i}d\sigma.
\end{equation}
Intuitively, rather than relying on the gradient at a single point, which may be uninformative or saturated, $\mathrm{IG}$ accumulates gradient information along the entire path from the baseline to the input, yielding a more robust and complete picture of each feature's contribution. To empirically obtain the $\mathrm{IG}$, an approximation can be computed on points occurring at sufficiently small intervals
\begin{equation}
    \widehat{\mathrm{IG}}_i(x) =(x_i - x_i')\times\frac{1}{M}\sum_{m=1}^M\frac{\partial F(\gamma(\frac{m}{M}))}{\partial x_i},
\end{equation}
where $m\in[50, 300]$. In practice, $m$ interpolation steps are sufficient to obtain a reliable approximation, with $[50, 300]$ being the standard range recommended in the original work.
Given the linearity of our last layer $L$, the IG of a subset of concepts $\mathcal{O}^{L-1}$ at layer $L-1$ can be expressed as
\begin{equation}
    \mathrm{IG}^{\mathrm{OPTIMUS}}_i(x) =(x_i - x_i')\times\frac{1}{M}\sum_{m=1}^M\sum_{k\in\mathcal{O}^{L-1}}w^L_{\hat{y},k}\times\frac{\partial z^{L-1}_k(\gamma(\frac{m}{M}))}{\partial x_i},
\end{equation}
where $w^L_{\hat{y},k}$ represent the weights of the last linear layer, predicting class $\hat{y}$. Restricting the sum to concepts $k\in\mathcal{O}^{L-1}$ rather than the full set of neurons effectively filters out irrelevant activations, ensuring that only the minimal and sufficient concepts contribute to the final heatmap.

\begin{figure}
    \centering
    \includegraphics[width=0.18\linewidth]{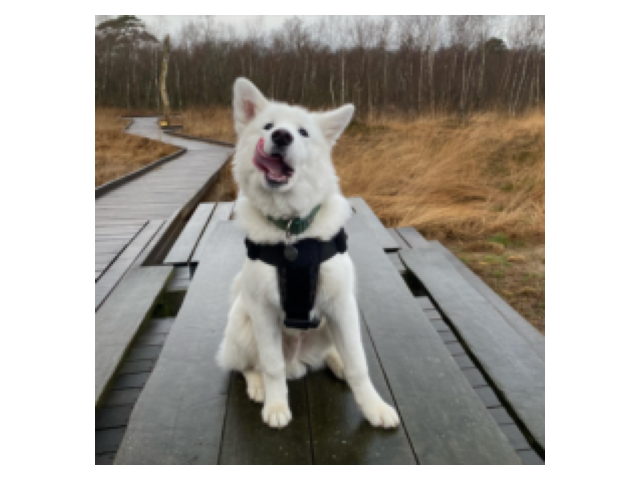}\hspace{-0.5cm}
    \includegraphics[width=0.18\linewidth]{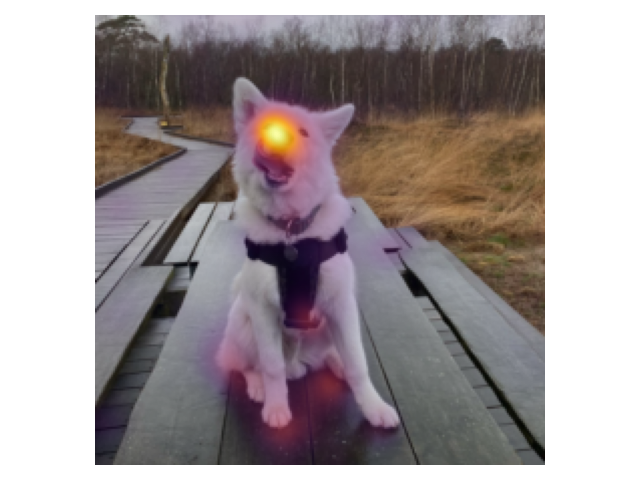}\hspace{-0.3cm}
    \includegraphics[width=0.18\linewidth]{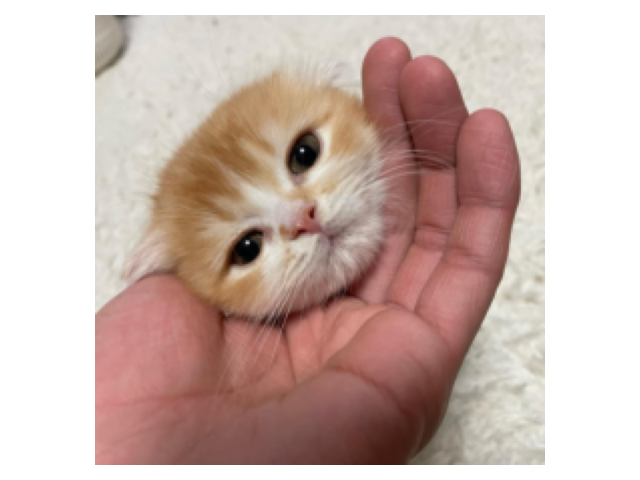}\hspace{-0.5cm}
    \includegraphics[width=0.18\linewidth]{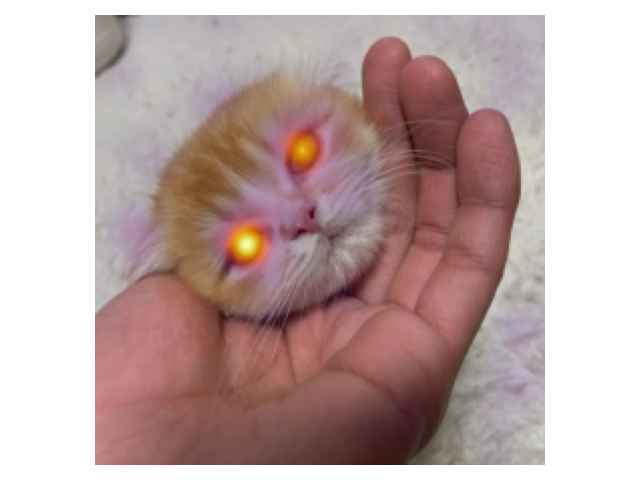}\hspace{-0.3cm}
    \includegraphics[width=0.18\linewidth]{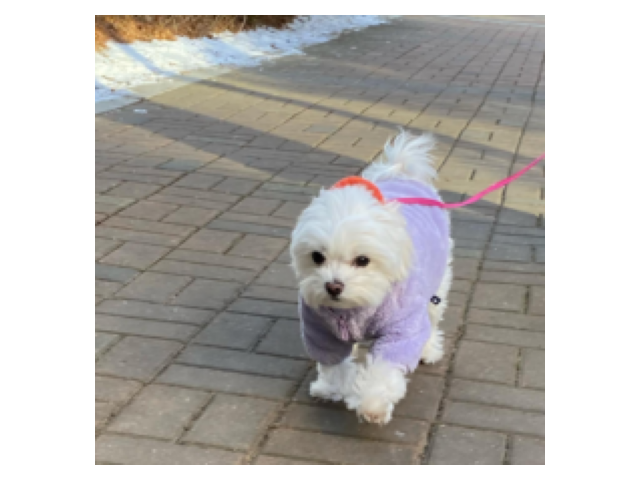}\hspace{-0.5cm}
    \includegraphics[width=0.18\linewidth]{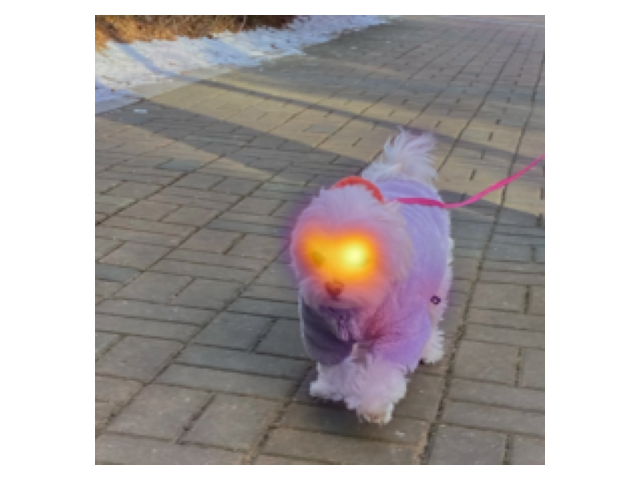}\\
    \includegraphics[width=0.18\linewidth]{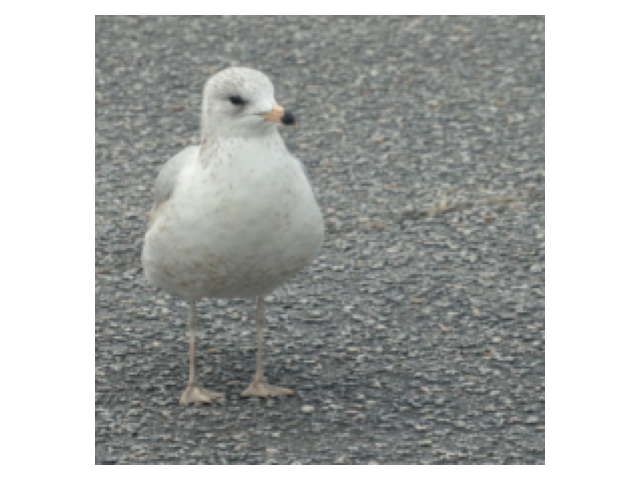}\hspace{-0.5cm}
    \includegraphics[width=0.18\linewidth]{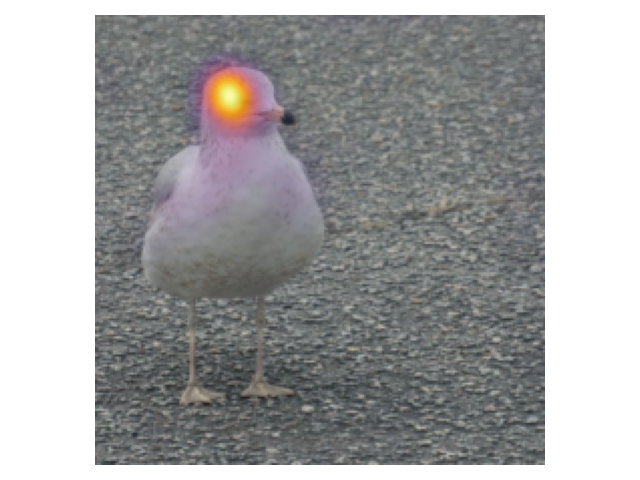}\hspace{-0.3cm}
    \includegraphics[width=0.18\linewidth]{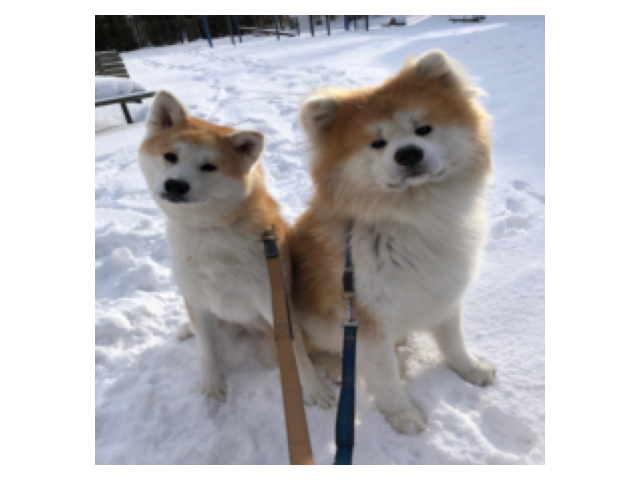}\hspace{-0.5cm}
    \includegraphics[width=0.18\linewidth]{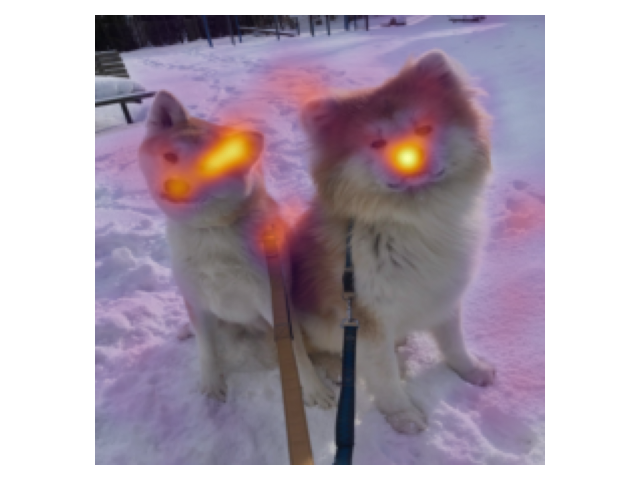}\hspace{-0.3cm}
    \includegraphics[width=0.18\linewidth]{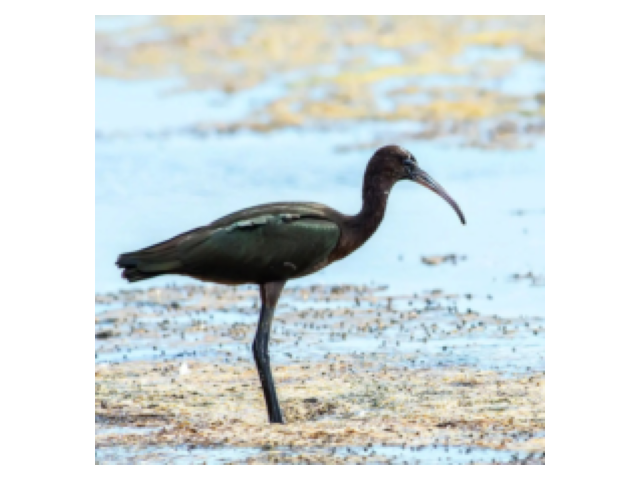}\hspace{-0.5cm}
    \includegraphics[width=0.18\linewidth]{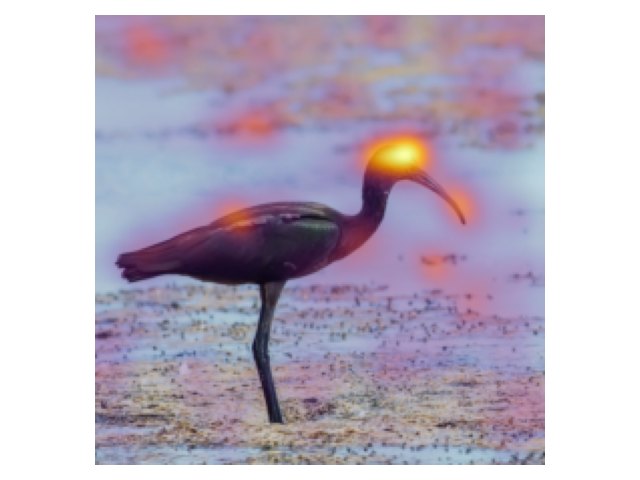}
    \caption{DeepLIFT-OPTIMUS explanations.}
    \label{fig:DL}
\end{figure}

\subsubsection{DeepLIFT}

While Integrated Gradients is computationally expensive due to its integration over $[0,1]$ and the repeated backpropagation of gradients, DeepLIFT proposes a single-point estimate instead. We first define $\Delta z^\ell_i = z_i^\ell(x) - z_i^\ell(x')$ as the activation difference with respect to the baseline $x'$ introduced earlier. A multiplier $m_{\Delta z^\ell_i\Delta z^{\ell + 1}_j}$ is then introduced as a finite-difference approximation of the partial derivative $\frac{\partial z_j^{\ell+1}}{\partial z_i^{\ell}}$ (when the baseline $x'$ is close to $x$), defined as follows
\begin{equation}
    m_{\Delta z^\ell_i\Delta z^{\ell + 1}_j} = \frac{\Delta z^{\ell+1}_j}{\Delta z^\ell_i}.
\end{equation}
A chain rule can then be applied to propagate the multipliers across the network
\begin{equation}
    m_{\Delta z^\ell_i\Delta z^{\ell + 2}_j} = \sum^{r^{\ell+1}}_t m_{\Delta z^\ell_i\Delta z^{\ell + 1}_t}\times m_{\Delta z^{\ell+1}_t\Delta z^{\ell + 2}_j}.
\end{equation}
The final attribution score is then obtained by backpropagating these multipliers through the entire network
\begin{equation}
    \mathrm{DeepLIFT}_i = (x_i - x_i')\times m_{\Delta z^0_i\Delta z^{L}_{\hat{y}}}.
\end{equation}
Once again, exploiting the linearity of our last layer $L$, the DeepLIFT scores restricted to a subset of concepts $\mathcal{O}^{L-1}$ at layer $L-1$ can be expressed as
\begin{equation}
    \mathrm{DeepLIFT}_i^{\mathrm{OPTIMUS}} = (x_i - x_i')\times\sum_{k\in\mathcal{O}^{L-1}}w^L_{\hat{y},k}\;m_{\Delta z^0_i\Delta z^{L-1}_{k}}.
\end{equation}
In this section, we showed that both DeepLIFT-OPTIMUS and IG-OPTIMUS can be expressed as a weighted sum over the previous layer's heatmap
\begin{equation}
    \mathrm{DeepLIFT}_i^{\mathrm{OPTIMUS}} = \sum_{k\in\mathcal{O}^{L-1}}w^L_{\hat{y},k}\times \mathrm{DeepLIFT}_k^{L-1},
\end{equation}
and
\begin{equation}
    \mathrm{IG}_i^{\mathrm{OPTIMUS}} = \sum_{k\in\mathcal{O}^{L-1}}w^L_{\hat{y},k}\times \widehat{\mathrm{IG}}_k^{L-1}.
\end{equation}
Here, the superscript indicates the target layer and the subscript identifies the target neuron with respect to which the attribution is computed. This formulation means that our implementation can directly leverage existing libraries to back-propagate concept-level attributions into the image space. The following section outlines the OPTIMUS-Prime framework applied to a practical visual classification task.

\section{Experiments}\label{sec:experiments}

For our experiments, we consider a Cat-Dog-Bird classification task~\cite{cdb-2026} with a $70/15/15$ train/val/test split. The encoder architecture is described below.
The following sections report results for DeepLIFT-OPTIMUS. An identical analysis for IG-OPTIMUS is provided in the Appendix.
Code is available at the link given on the first page of the article.

\subsection{Experimental setup}

The encoder maps $x \in \mathbb{R}^{3 \times 224 \times 224}$ to $z \in \mathbb{R}^{512 \times 14 \times 14}$ through four Conv-BN-ReLU-MaxPool blocks (channels: 32, 64, 128, 256), followed by two Conv-BN blocks (512 channels, stride 1) with a final Sigmoid activation. All convolutions use $3\times3$ kernels with padding 1.
During training, the backbone is followed by adaptive average pooling, a dropout layer ($p=0.3$), and two fully-connected layers (512$\to$256 with ReLU, then 256$\to$3) for 3-class classification.
At test time, the backbone is frozen and replaced by a single linear probe ($512 \to 3$) preceded by adaptive average pooling.
The model is optimized with cross-entropy loss and Adam ($\text{lr}=10^{-3}$), with a batch size of 64, trained for 20 epochs on the full network and 10 additional epochs for the linear probe.

\subsection{DeepLIFT-OPTIMUS}

Once the model is trained, a PI can be computed (cf.~Section~\ref{sec:prime}) on the final linear layer, whose inputs are the last backbone features bounded in $[0,1]$ by the Sigmoid activation. As shown in Section~\ref{sec:heatmaps}, DeepLIFT-OPTIMUS can be expressed as a weighted sum of DeepLIFT attributions over the preceding layer. We therefore use the neuron-wise DeepLIFT implementation from Captum~\cite{captum_2020}.

Figure~\ref{fig:DL} presents our OPTIMUS explanations. The original image is shown on the left, while the heatmap is on the right, ranging from transparent purple to yellow, highlighting the sufficient and minimal concepts required to
sustain the model's prediction.

Figure~\ref{fig:DL-opt-vs-diff} illustrates the difference between the full neuron-layer activation and the OPTIMUS explanation, highlighting unnecessary concepts. Depending on the concepts captured by each neuron, non-OPTIMUS concepts may partially overlap with the selected ones. The left panel shows the OPTIMUS explanation as described above; the right panel displays the difference between all concept activations and the OPTIMUS subset, highlighted in mauve.
Finally, Figure~\ref{fig:DL-diff} provides additional visualizations of unnecessary concepts, as defined above.

\begin{figure}
    \centering
    \includegraphics[width=0.18\linewidth]{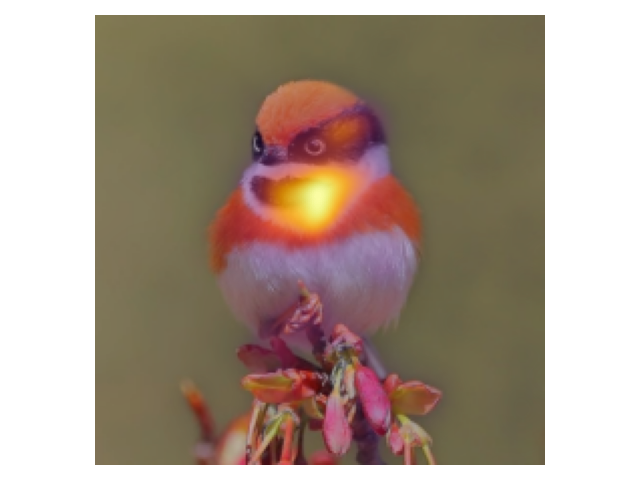}\hspace{-0.5cm}
    \includegraphics[width=0.18\linewidth]{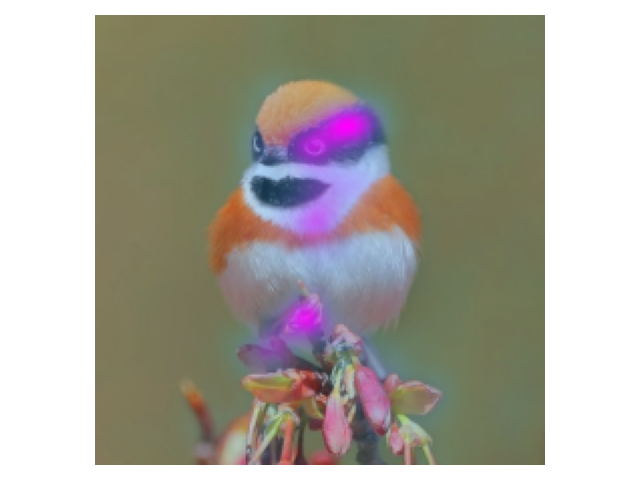}\hspace{-0.3cm}
    \includegraphics[width=0.18\linewidth]{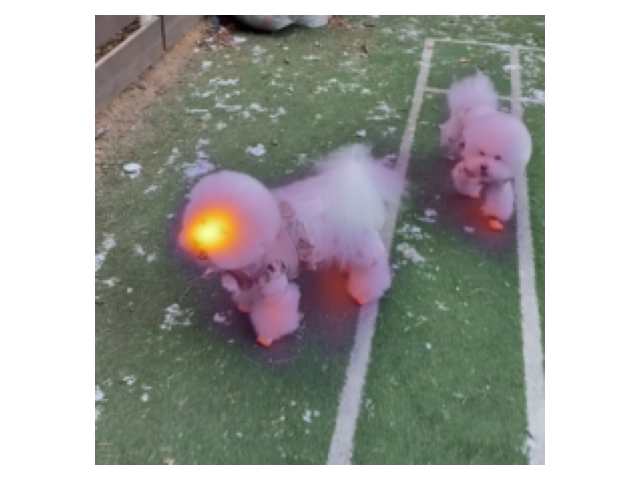}\hspace{-0.5cm}
    \includegraphics[width=0.18\linewidth]{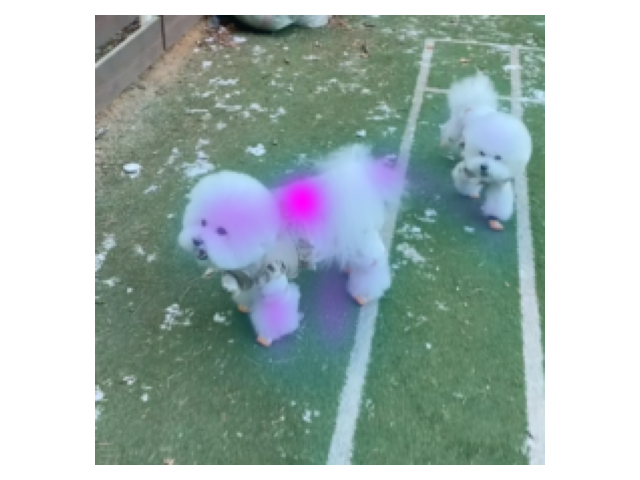}\hspace{-0.3cm}
    \includegraphics[width=0.18\linewidth]{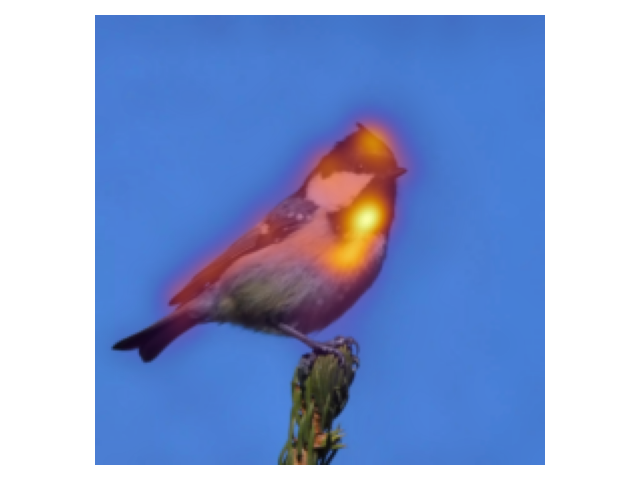}\hspace{-0.5cm}
    \includegraphics[width=0.18\linewidth]{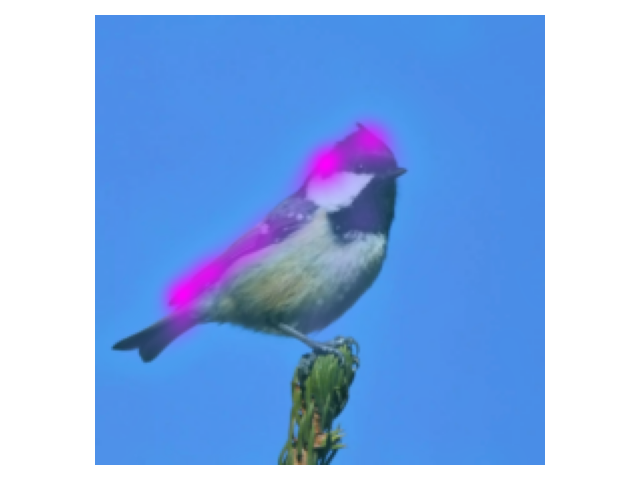}
    \caption{On the left: DeepLIFT-OPTIMUS. On the right: the difference between full concepts highlights and OPTIMUS (positive difference in mauve), highlighting unnecessary concepts.}
    \label{fig:DL-opt-vs-diff}
\end{figure}

\subsection{Analytic evaluation of OPTIMUS explanations}
To complement the qualitative visualizations, we evaluate the computation of OPTIMUS explanations in terms of PI size and runtime. We compare the two proposed greedy rules, \textsc{LowGain} and \textsc{MaxMin}, against a \textsc{Unordered} baseline, which removes features according to their original index order and does not use any margin-based criterion. Detailed results are reported in Appendix~\ref{app:analytic-exp}.

Overall, \textsc{LowGain} provides the best trade-off between compactness (length) and efficiency: it typically yields PI sizes close to \textsc{MaxMin}, while being substantially faster in high-dimensional settings. 
This makes it a practical heuristic for repeated PI computations, including enumeration-oriented analyses. 
The \textsc{Unordered} baseline is often fastest but produces larger PIs, whereas \textsc{MaxMin} remains a useful compactness-oriented reference for optimization-based searches. The artificial-data results further characterize how PI size and runtime scale with the number of variables and classes, providing a benchmark for future comparisons.

\section{Discussion}\label{sec:discussion}

OPTIMUS-Prime remains partially \emph{model-specific}, as a few assumptions must hold. The first constraint is the boundedness of the concept space at the targeted layer. As discussed in Section~\ref{sec:prime}, this requirement can be relaxed at the cost of trading formal guarantees for PAC statistical ones. The second, non-removable assumption is the linearity of the targeted layer. While final layers are commonly linear in practice, this prevents OPTIMUS-Prime from being applicable to arbitrary \emph{black-box} models. Additional minor constraints arise from the use of DeepLIFT~\cite{shrikumar2017} and Integrated Gradients~\cite{sundararajan2017}, whose applicability similarly depends on the underlying architecture. Furthermore, recent work has raised concerns about the reliability of gradient-based attribution methods, highlighting sensitivity to baseline choice and input perturbations~\cite{adebayo2018}, limitations that our formal PI framework partially addresses by grounding explanations in provable sufficiency rather than gradient sensitivity.

A recurring observation from our experiments is that multiple neurons can encode the same concept~\cite{li2015}, a phenomenon partly attributable to the simplicity of the classification tasks considered. While this redundancy is harmless in practice, one may nonetheless prefer disentangled concept representations. This can typically be addressed at inference time through a variety of existing techniques~\cite{kornblith2019, elhage2022, koh_concept_2020}, which we leave as a natural extension of the present framework.

Finally, we reflect on the semantic meaning of latent concepts, which can sometimes remain too abstract to be directly interpretable by humans. This is precisely why we advocate for input-space heatmaps as the primary medium for communicating explanations. Prior work has shown that neurons can be inherently interpretable across a range of architectures~\cite{karpathy2015, bau_network_2017, bau2020}. For instance, word embedding concepts appear to encode semantic directions~\cite{mikolov2013}, while superposition phenomena have been documented in more recent settings~\cite{elhage2022, cheung2019}. We deliberately leave these challenges to the practitioner, in keeping with our goal of designing a framework that is as \emph{model-agnostic} as possible, and that can be naturally composed with existing interpretability tools to address architecture-specific constraints.

\begin{figure}
    \centering
    \includegraphics[width=0.18\linewidth]{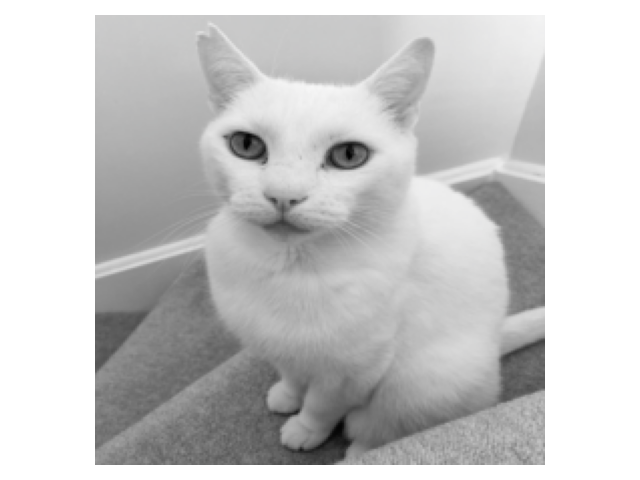}\hspace{-0.5cm}
    \includegraphics[width=0.18\linewidth]{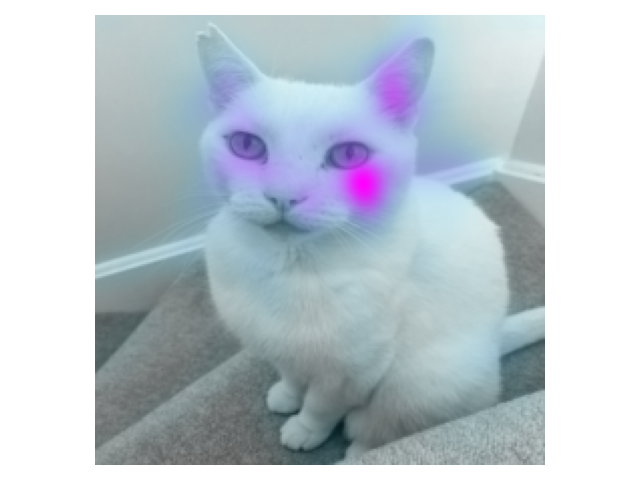}\hspace{-0.3cm}
    \includegraphics[width=0.18\linewidth]{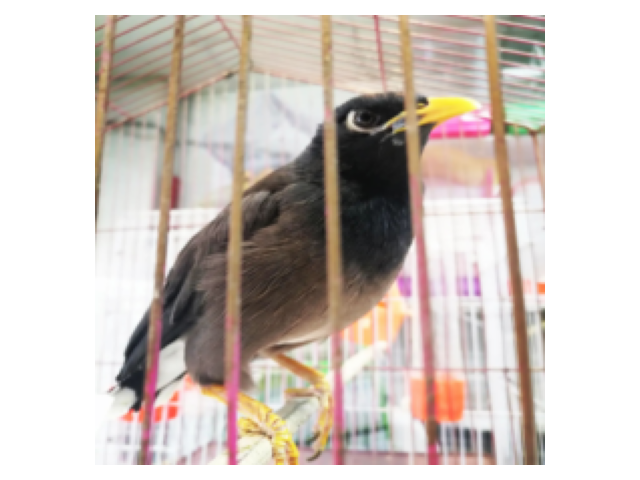}\hspace{-0.5cm}
    \includegraphics[width=0.18\linewidth]{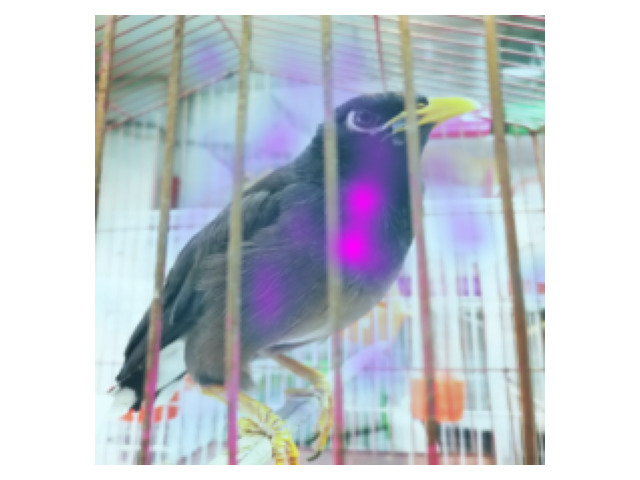}\hspace{-0.3cm}
    \includegraphics[width=0.18\linewidth]{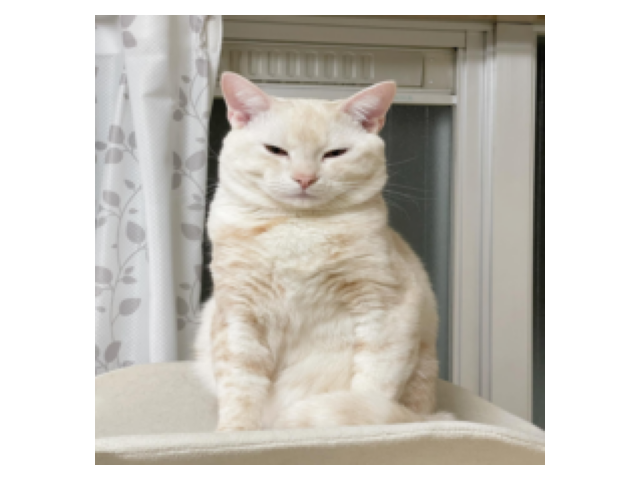}\hspace{-0.5cm}
    \includegraphics[width=0.18\linewidth]{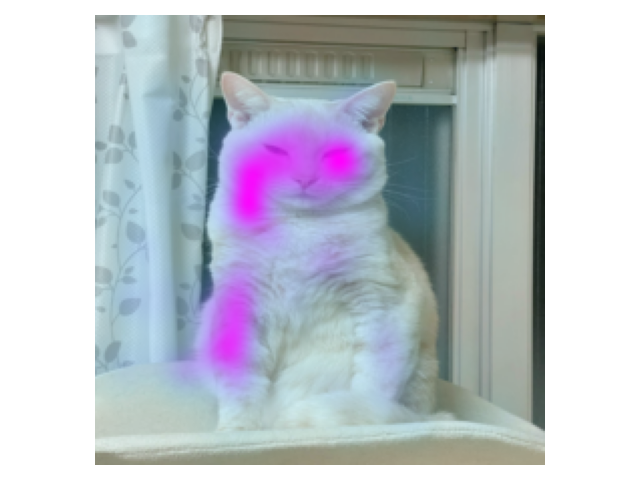}\\
    \includegraphics[width=0.18\linewidth]{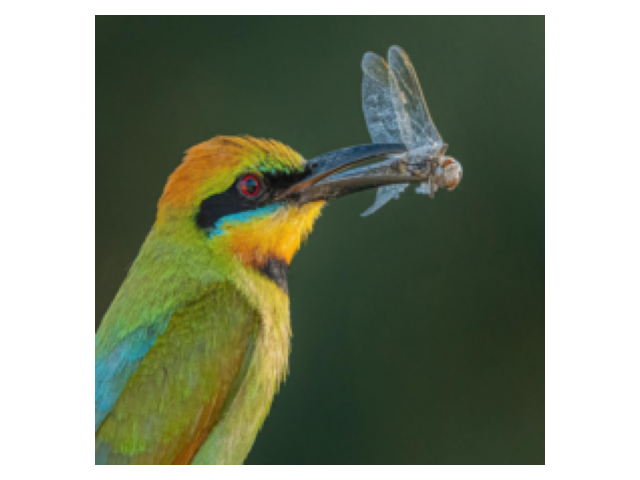}\hspace{-0.5cm}
    \includegraphics[width=0.18\linewidth]{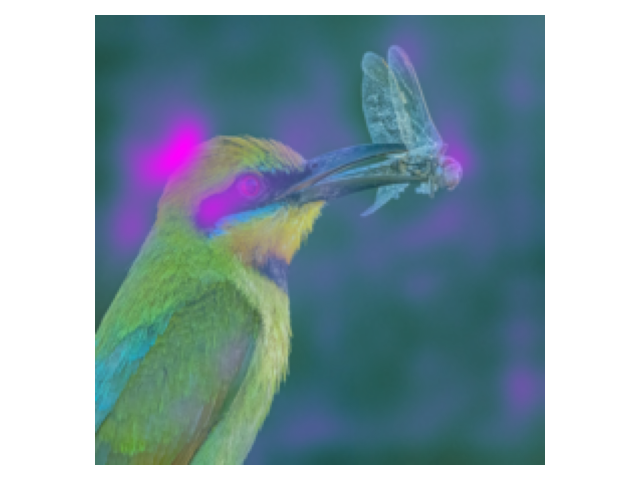}\hspace{-0.3cm}
    \includegraphics[width=0.18\linewidth]{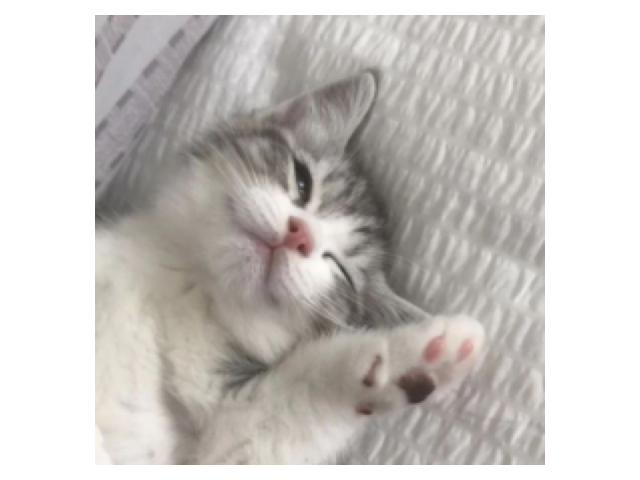}\hspace{-0.5cm}
    \includegraphics[width=0.18\linewidth]{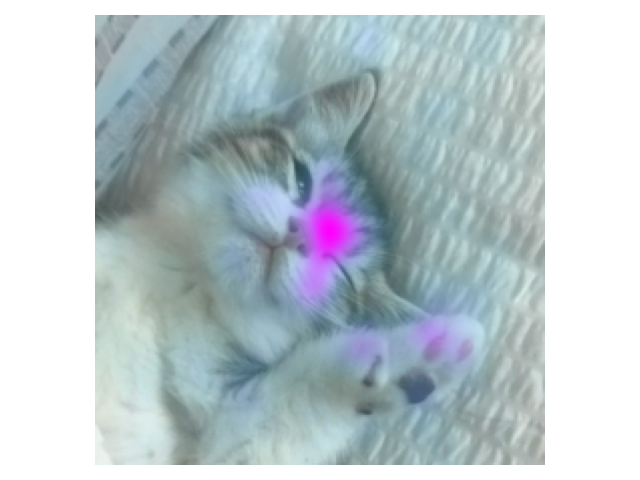}\hspace{-0.3cm}
    \includegraphics[width=0.18\linewidth]{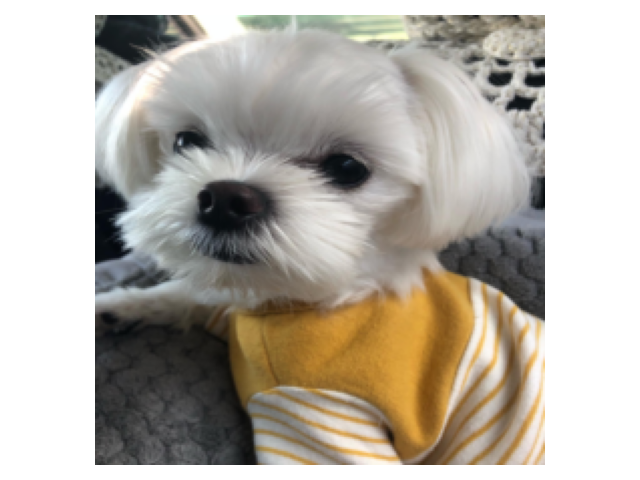}\hspace{-0.5cm}
    \includegraphics[width=0.18\linewidth]{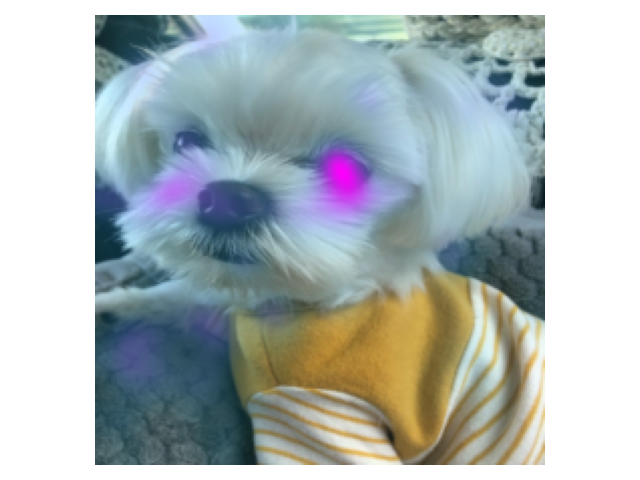}
    \caption{Difference between DeepLIFT and DeepLIFT-OPTIMUS: unnecessary concepts.}
    \label{fig:DL-diff}
\end{figure}

\section{Conclusion}\label{sec:conclusion}

In this paper, we present OPTIMUS-Prime, a new framework for deep visual explanations introducing OPTIMUS explanations. These explanations are grounded in theoretical derivations that provide strong guarantees while maintaining high interpretability: the highlighted concepts are both sufficient and minimal to sustain the model's prediction.

Although OPTIMUS-Prime is broadly and almost transparently applicable to many deep architectures, the framework remains subject to a few assumptions. Notably, the targeted layer must be linear, and the preceding layer's output must be bounded to preserve our theoretical guarantees. In practice, the semantic quality of the explanations further depends on the information captured within the targeted neurons, which may be entangled or superposed. We therefore leave users free to design alternative architectures that best suit their needs.
A natural and promising extension of this work would be the application of OPTIMUS-Prime to recent transformer-based architectures, which we leave for future work.

\bibliographystyle{plain} 
\bibliography{refs.bib}

\newpage
\appendix
\setcounter{proposition}{0}

\section{Notations, Acronyms and Datasets}\label{appendix:Notations}

\begin{table}[h]
\centering
\begin{tabular}{ll}
    \hline
        \textbf{Notations/Acronyms} & \textbf{Description} \\
    \hline
        $\mathcal{X}$ & Input space $\mathcal{X} \subseteq \mathbb{R}^P$ \\
        $P$ & Input dimension \\
        $\mathcal{Y}$ & Label space $\{1, \dots, K\}$ \\
        $(x_n, y_n)$ & i.i.d.\ draw from the joint distribution $P_{X,Y}$ \\
        $\mathcal{D}$ & Dataset $\{(x_n, y_n)\}_{n=1}^N$ \\
        $N$ & Number of training instances \\
    \hline
        $h$ & Neural network classifier $h : \mathcal{X} \to \mathcal{Y}$ \\
        $L$ & Number of layers in the network \\
        $K$ & Number of classes \\
        $\hat{y}(x)$ & Predicted class for input $x$ \\
        $z^{\ell}(x)$ & Output of the $\ell$-th hidden layer, $z^{\ell} \in \mathbb{R}^{r^\ell}$ \\
        $r^\ell$ & Number of neurons at layer $\ell$ \\
    \hline
        $\circ^\ell$ & A concept index in $\{1, \dots, r^\ell\}$ at layer $\ell$ \\
        $\mathcal{O}^{\ell}$ & OPTIMUS prime implicant (set of concepts) at layer $\ell$ \\
        $\mathbf{m}$ & OPTIMUS explanation heatmap\\
    \hline
        $f$ & Linear multiclass classifier function $f : \mathcal{Z} \to \mathbb{R}^K$ \\
        $x^*$ & A given instance with predicted class $c$ \\
        $I$ & Set of all feature indices $\{1, \dots, r\}$ \\
        $S$ & Candidate subset of features $S \subseteq I$ \\
        $\Gamma_k(S)$ & Worst-case dominance margin for class $k$ given fixed features $S$ \\
        $g_q^{(k)}$ & Margin gain of feature $q$ w.r.t.\ rival class $k$ \\
    \hline
        $x'$ & Baseline input (uniformly black image) \\
        $\gamma(\sigma)$ & Linear path between baseline $x'$ and input $x$ \\
        $\mathrm{IG}_i(x)$ & Integrated Gradients attribution at pixel $i$ \\
        $\widehat{\mathrm{IG}}_i(x)$ & Empirical approximation of $\mathrm{IG}_i(x)$ with $M$ steps \\
        $M$ & Number of interpolation steps for IG (typically in $[50, 300]$) \\
        $\mathrm{IG}_i^{\mathrm{OPTIMUS}}(x)$ & OPTIMUS-restricted Integrated Gradients at pixel $i$ \\
        $\mathrm{DeepLIFT}_i$ & DeepLIFT attribution score at pixel $i$ \\
        $\mathrm{DeepLIFT}_i^{\mathrm{OPTIMUS}}$ & OPTIMUS-restricted DeepLIFT attribution at pixel $i$ \\
        $w^L_{\hat{y},k}$ & Weight of the last linear layer for predicted class $\hat{y}$ at neuron $k$ \\
    \hline\\
\end{tabular}
\caption{Notations and Acronyms}\label{tab:Notations and Acronyms}
\end{table}

\section{Proof for Proposition~\ref{prop:direct_search_pi}}\label{appendix:direct_search_pi}
\begin{proposition}
Assume that $x^*$ is classified as class $c$. Starting from $S=I$, by repeatedly removing any feature $q\in S$ such that \(\forall k\neq c \quad \Gamma_k(S\setminus\{q\})>0\), the resulting set $S$ is a multiclass PI when no further removal can guarantee the constraint.
\end{proposition}
\begin{proof}
Since $\Gamma_k(I)>0$ for all $k\neq c$, the initial set $I$ is a multiclass implicant. At each step, a feature is removed only if the remaining set still satisfies \(\Gamma_k(S)>0,\ \forall k\neq c\). Thus, the implicant property is preserved throughout the procedure. At termination, no feature $q\in S$ can be removed while preserving all dominance constraints. Hence, for every $q\in S$, there exists some rival
class $k\neq c$ such that
\(
\Gamma_k(S\setminus\{q\})\leq 0.
\)
Therefore, no subset obtained by deleting one more feature remains an implicant, which means that $S$ is subset-minimal, and hence a multiclass PI.
\end{proof}

\section{Algorithm for finding a multiclass PI}\label{alg:find_multiclass_pi}
\begin{algorithm}[H]
\caption{Finding one multiclass PI.}
\label{alg:ordered-greedy-pruning-pi}
\KwIn{Feature set \(I\), feature dimension $r$, predicted class \(c\), and greedy rule \(\mathcal G\).}
\KwOut{A multiclass PI \(S\subseteq I\).}

Initialize \(S \leftarrow I\)

Sort the features in ascending order according to \(\mathcal G\):
\(
(q_{(1)},\ldots,q_{(r)}) \leftarrow \operatorname{Order}_{\mathcal G}(I)
\)

\For{\(i=1,\ldots, r\)}{
    \If{\(\;q_{(i)}\in E\;\) and \(\;\Gamma_k(S\setminus\{q_{(i)}\})>0\), $\; \forall k\neq c\;$}{
        \(S \leftarrow S\setminus\{q_{(i)}\}\)
    }
}

\Return{\(S\)}
\end{algorithm}

\section{Additional analytic evaluation}
\label{app:analytic-exp}

Tables~\ref{tab:pi-real} report the average size and computation time of the PIs obtained by each greedy rule. Lower PI size indicates a more compact explanation, while lower runtime indicates a more efficient search procedure. Bold values indicate the shortest PI size or the quickest computation time within each dataset. The results show that the same trend persists in this regime: \textsc{MaxMin} gives the smallest PIs, \textsc{Unordered} is the fastest, and \textsc{LowGain} achieves an intermediate trade-off between PI length and computation time.

For the image embeddings, the number of variables is much larger, as each instance is represented by a 512-dimensional feature vector extracted from the image model. This setting therefore provides a higher-dimensional evaluation scenario, closer to the latent representations used in deep vision models. 

\begin{table}
\centering
\small
\caption{Analytic comparison of multiclass PI search on real data and image embeddings.}
\label{tab:pi-real}
\begin{adjustbox}{max width=\textwidth}
\begin{tabular}{lrrlrr}
\toprule
Dataset & Classes & Variables & Rules & PI size & Time (ms) \\
\midrule
Iris & 3 & 4 & \textsc{MaxMin} 
& \(\mathbf{3.333 \pm 0.516}\) & 0.046 \\
Iris & 3 & 4 & \textsc{LowGain} 
& \(\mathbf{3.333 \pm 0.516}\) & 0.055 \\
Iris & 3 & 4 & \textsc{Unordered} 
& \(\mathbf{3.333 \pm 0.516}\) & \(\mathbf{0.042}\) \\
\midrule
Wine & 3 & 13 & \textsc{MaxMin} 
& \(\mathbf{8.315 \pm 1.537}\) & 0.220 \\
Wine & 3 & 13 & \textsc{LowGain} 
& \(\mathbf{8.315 \pm 1.537}\) & 0.094 \\
Wine & 3 & 13 & \textsc{Unordered} 
& \(9.796 \pm 1.192\) & \(\mathbf{0.081}\) \\
\midrule
Digits & 10 & 64 & \textsc{MaxMin} 
& \(\mathbf{29.656 \pm 4.337}\) & 4.864 \\
Digits & 10 & 64 & \textsc{LowGain} 
& \(30.144 \pm 4.350\) & 0.371 \\
Digits & 10 & 64 & \textsc{Unordered} 
& \(37.619 \pm 4.258\) & \(\mathbf{0.290}\) \\
\midrule
Image embeddings & 3 & 512 & \textsc{MaxMin} 
& \(\mathbf{353.814 \pm 41.086}\) & 330.997 \\
Image embeddings & 3 & 512 & \textsc{LowGain} 
& \(371.986 \pm 43.110\) & 6.203 \\
Image embeddings & 3 & 512 & \textsc{Unordered} 
& \(434.436 \pm 27.962\) & \(\mathbf{4.972}\) \\
\bottomrule
\end{tabular}
\end{adjustbox}
\end{table}

Figures~\ref{fig:pi-comp-4}--\ref{fig:pi-comp-12} report PI size and computation time as functions of the number of variables on artificially generated data. This setting allows us to vary both the feature dimension and the number of classes. For each number of classes, the left plot shows the average PI size with standard deviation as a shaded region, while the right plot shows the average computation time.

As expected, PI size and runtime generally increase with the number of variables. The \textsc{Unordered} rule is consistently faster, but usually returns larger PIs. In contrast, \textsc{MaxMin} and \textsc{LowGain} produce PI sizes that remain relatively close, while \textsc{MaxMin} becomes substantially more expensive as the feature dimension grows.

We also observe that \textsc{MaxMin} tends to run faster when the number of classes increases, for a fixed number of variables. A possible explanation is that additional classes impose tighter dominance constraints, leaving fewer removable features and making the choice of \(q\) more efficient, although this typically leads to larger PIs.
The raw numerical results are reported in Table~\ref{tab:pi-artificial}.

\begin{figure}[p]
    \centering
    \includegraphics[width=0.8\linewidth]{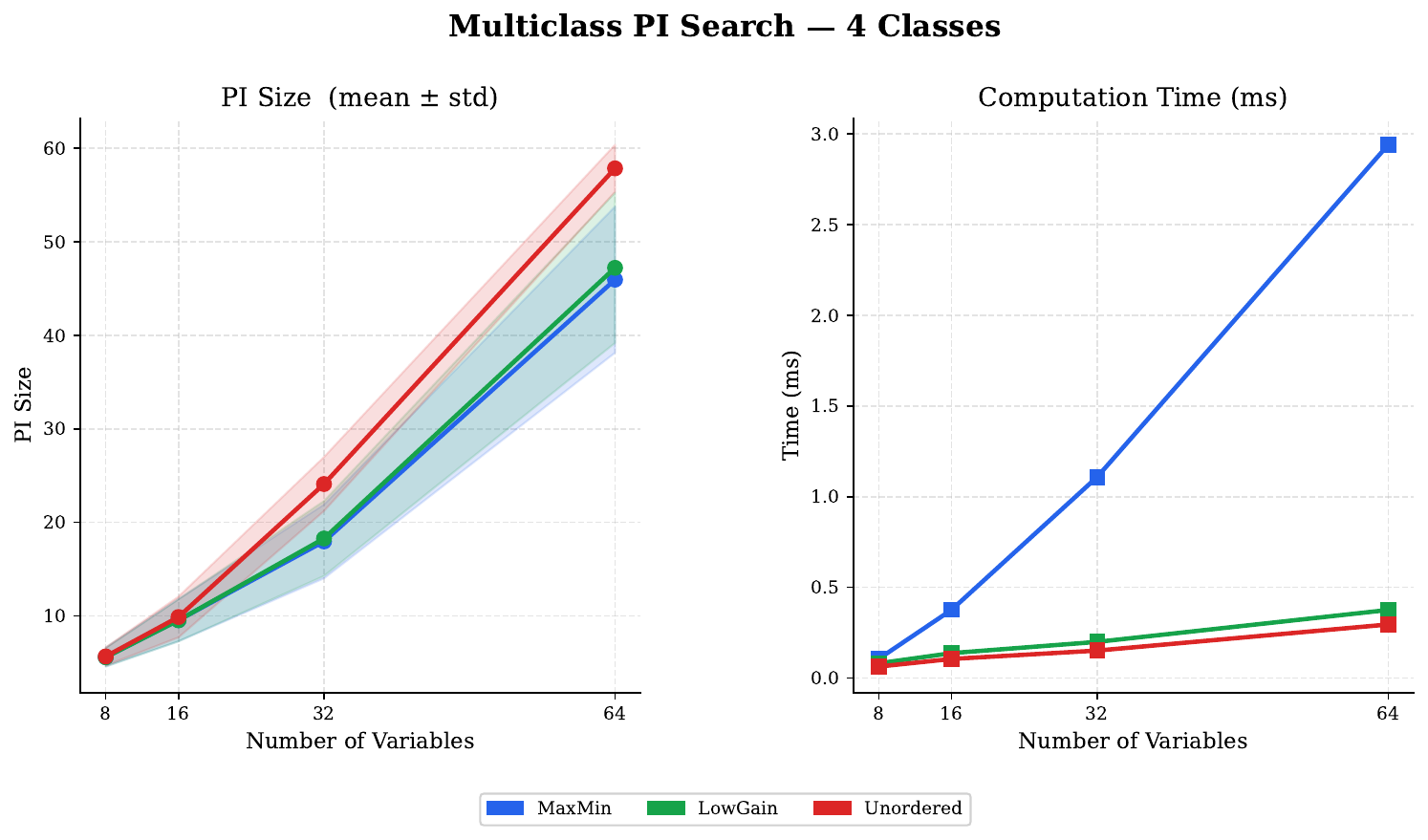}
    \caption{Artificial-data comparison of 4-class PI search \emph{vs.} number of variables.}
    \label{fig:pi-comp-4}
\end{figure}

\begin{figure}[p]
    \centering
    \includegraphics[width=0.8\linewidth]{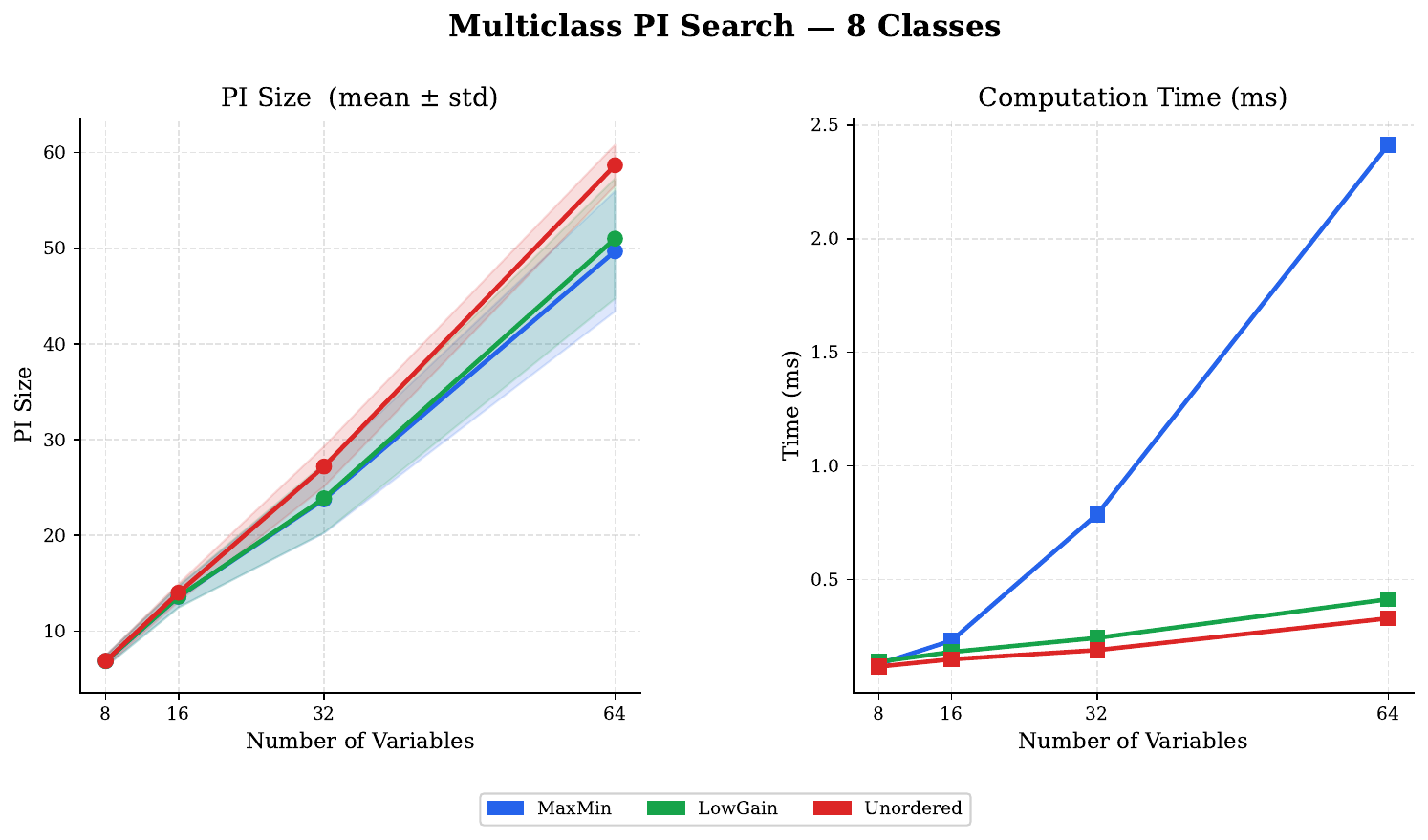}
    \caption{Artificial-data comparison of 8-class PI search \emph{vs.} number of variables.}
    \label{fig:pi-comp-8}
\end{figure}

\begin{figure}[p]
    \centering
    \includegraphics[width=0.8\linewidth]{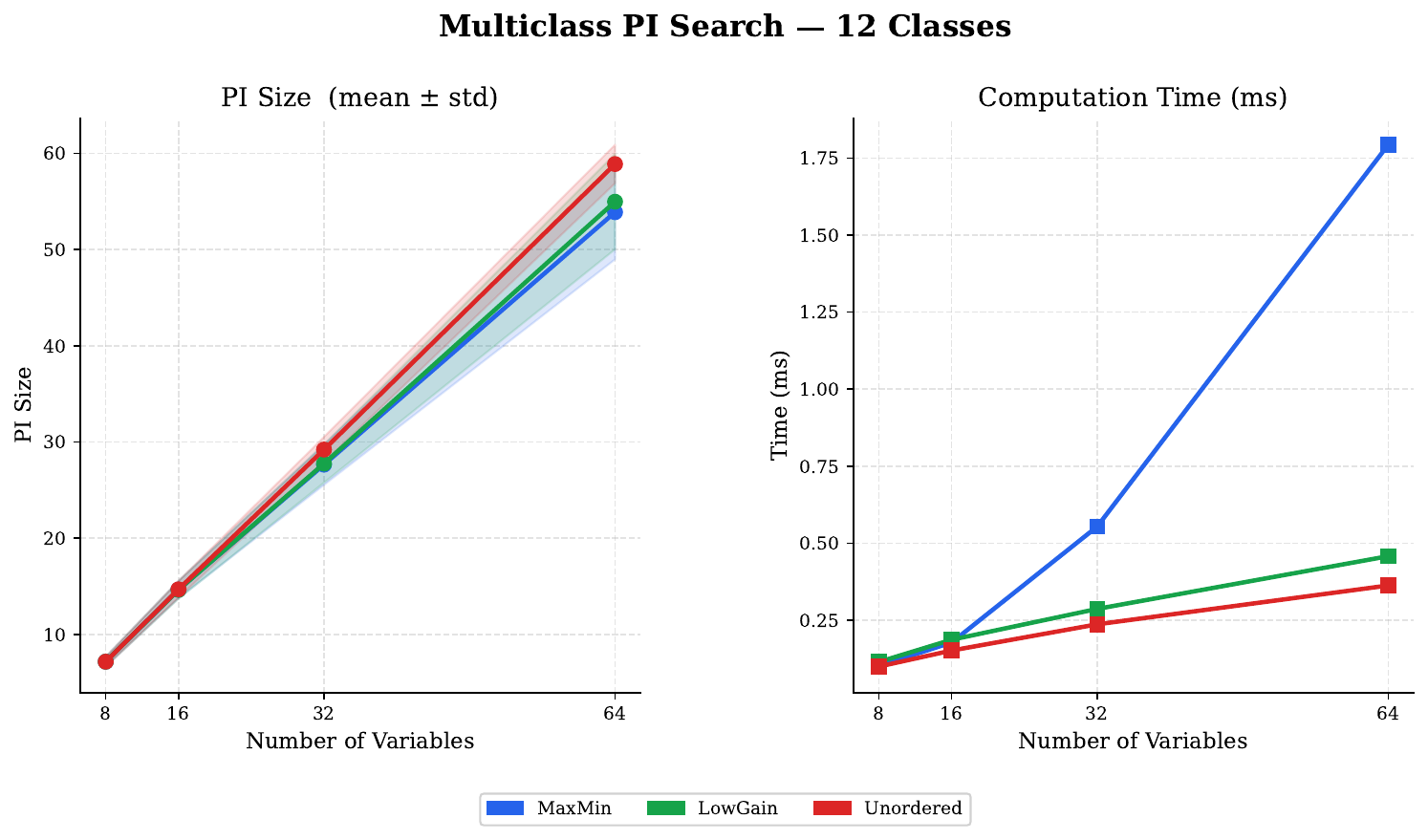}
    \caption{Artificial-data comparison of 12-class PI search \emph{vs.} number of variables.}
    \label{fig:pi-comp-12}
\end{figure}

\begin{table}
\centering
\small
\caption{Artificial-data comparison of multiclass PI search as a function of the number of classes and variables.}
\label{tab:pi-artificial}
\begin{adjustbox}{max width=\textwidth}
\begin{tabular}{rrlrr}
\toprule
Classes & Variables & Rules & PI size & Time (ms) \\
\midrule
4 & 8 & \textsc{MaxMin} 
& \(5.575 \pm 0.997\) & 0.106 \\
4 & 8 & \textsc{LowGain} 
& \(\mathbf{5.567 \pm 0.998}\) & 0.081 \\
4 & 8 & \textsc{Unordered} 
& \(5.650 \pm 0.972\) & \(\mathbf{0.063}\) \\
\midrule
4 & 16 & \textsc{MaxMin} 
& \(9.525 \pm 2.254\) & 0.376 \\
4 & 16 & \textsc{LowGain} 
& \(\mathbf{9.517 \pm 2.258}\) & 0.137 \\
4 & 16 & \textsc{Unordered} 
& \(9.867 \pm 2.183\) & \(\mathbf{0.105}\) \\
\midrule
4 & 32 & \textsc{MaxMin} 
& \(\mathbf{17.950 \pm 3.951}\) & 1.108 \\
4 & 32 & \textsc{LowGain} 
& \(18.267 \pm 4.010\) & 0.199 \\
4 & 32 & \textsc{Unordered} 
& \(24.092 \pm 2.890\) & \(\mathbf{0.151}\) \\
\midrule
4 & 64 & \textsc{MaxMin} 
& \(\mathbf{45.950 \pm 7.822}\) & 2.940 \\
4 & 64 & \textsc{LowGain} 
& \(47.217 \pm 8.039\) & 0.375 \\
4 & 64 & \textsc{Unordered} 
& \(57.858 \pm 2.521\) & \(\mathbf{0.295}\) \\
\midrule
8 & 8 & \textsc{MaxMin} 
& \(\mathbf{6.875 \pm 0.571}\) & 0.128 \\
8 & 8 & \textsc{LowGain} 
& \(\mathbf{6.875 \pm 0.571}\) & 0.138 \\
8 & 8 & \textsc{Unordered} 
& \(6.900 \pm 0.554\) & \(\mathbf{0.117}\) \\
\midrule
8 & 16 & \textsc{MaxMin} 
& \(\mathbf{13.592 \pm 1.107}\) & 0.231 \\
8 & 16 & \textsc{LowGain} 
& \(13.600 \pm 1.150\) & 0.181 \\
8 & 16 & \textsc{Unordered} 
& \(14.017 \pm 0.931\) & \(\mathbf{0.149}\) \\
\midrule
8 & 32 & \textsc{MaxMin} 
& \(\mathbf{23.767 \pm 3.528}\) & 0.787 \\
8 & 32 & \textsc{LowGain} 
& \(23.883 \pm 3.601\) & 0.243 \\
8 & 32 & \textsc{Unordered} 
& \(27.200 \pm 2.096\) & \(\mathbf{0.189}\) \\
\midrule
8 & 64 & \textsc{MaxMin} 
& \(\mathbf{49.683 \pm 6.284}\) & 2.413 \\
8 & 64 & \textsc{LowGain} 
& \(51.008 \pm 6.275\) & 0.413 \\
8 & 64 & \textsc{Unordered} 
& \(58.675 \pm 2.134\) & \(\mathbf{0.329}\) \\
\midrule
12 & 8 & \textsc{MaxMin} 
& \(\mathbf{7.167 \pm 0.506}\) & 0.102 \\
12 & 8 & \textsc{LowGain} 
& \(\mathbf{7.167 \pm 0.506}\) & 0.115 \\
12 & 8 & \textsc{Unordered} 
& \(\mathbf{7.167 \pm 0.506}\) & \(\mathbf{0.100}\) \\
\midrule
12 & 16 & \textsc{MaxMin} 
& \(14.650 \pm 0.900\) & 0.178 \\
12 & 16 & \textsc{LowGain} 
& \(\mathbf{14.633 \pm 0.921}\) & 0.187 \\
12 & 16 & \textsc{Unordered} 
& \(14.692 \pm 0.874\) & \(\mathbf{0.152}\) \\
\midrule
12 & 32 & \textsc{MaxMin} 
& \(\mathbf{27.650 \pm 2.104}\) & 0.554 \\
12 & 32 & \textsc{LowGain} 
& \(27.758 \pm 1.983\) & 0.287 \\
12 & 32 & \textsc{Unordered} 
& \(29.217 \pm 1.311\) & \(\mathbf{0.237}\) \\
\midrule
12 & 64 & \textsc{MaxMin} 
& \(\mathbf{53.892 \pm 4.905}\) & 1.793 \\
12 & 64 & \textsc{LowGain} 
& \(54.983 \pm 4.956\) & 0.458 \\
12 & 64 & \textsc{Unordered} 
& \(58.900 \pm 2.002\) & \(\mathbf{0.363}\) \\
\bottomrule
\end{tabular}
\end{adjustbox}
\end{table}
\section{Experimental results with Integrated Gradients}

\begin{figure}[H]
    \centering
    \includegraphics[width=0.18\linewidth]{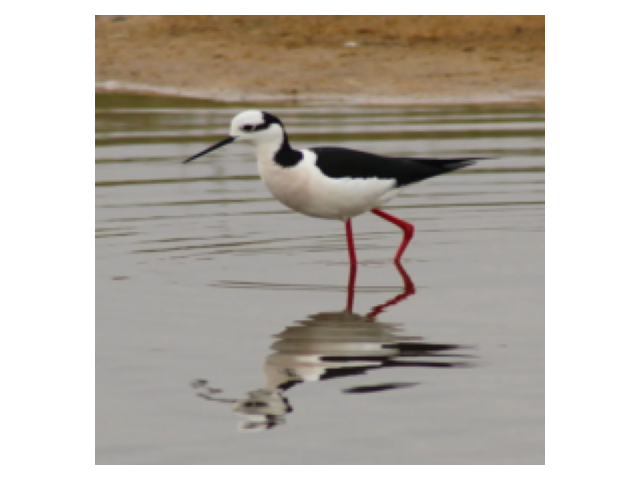}\hspace{-0.5cm}
    \includegraphics[width=0.18\linewidth]{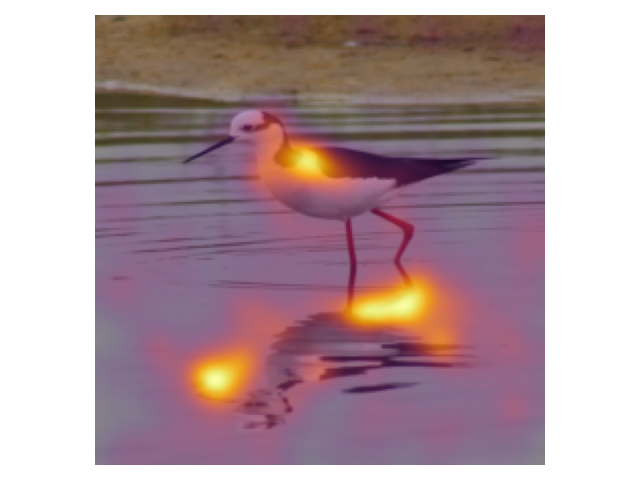}\hspace{-0.3cm}
    \includegraphics[width=0.18\linewidth]{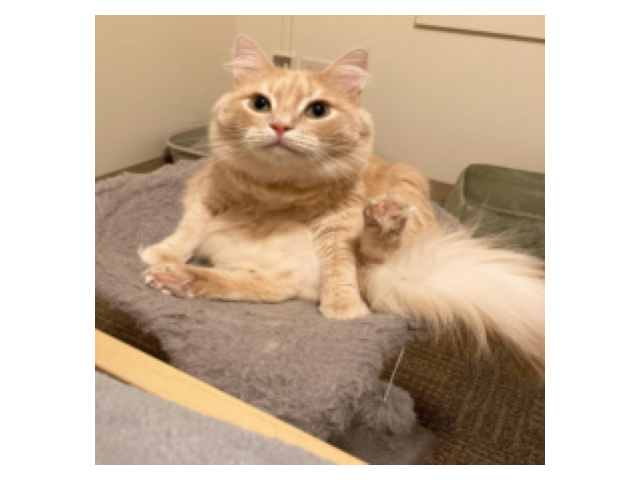}\hspace{-0.5cm}
    \includegraphics[width=0.18\linewidth]{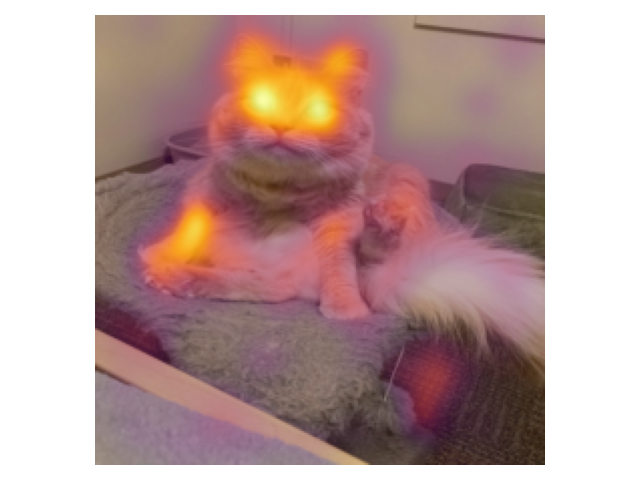}\hspace{-0.3cm}
    \includegraphics[width=0.18\linewidth]{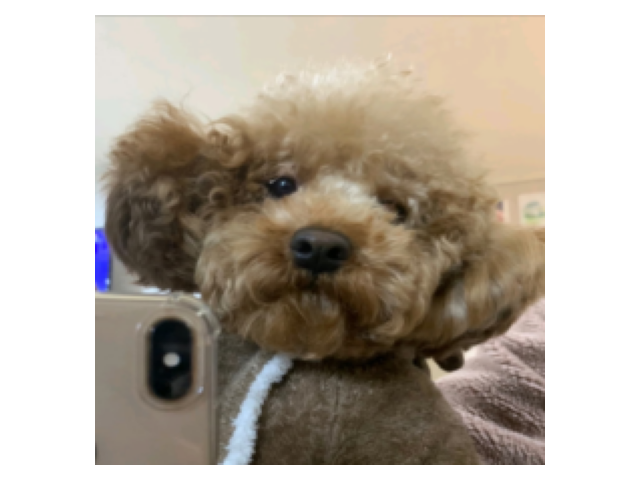}\hspace{-0.5cm}
    \includegraphics[width=0.18\linewidth]{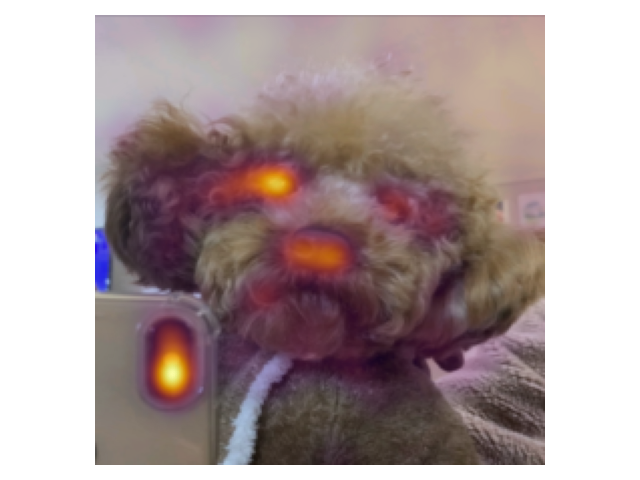}\\
    \includegraphics[width=0.18\linewidth]{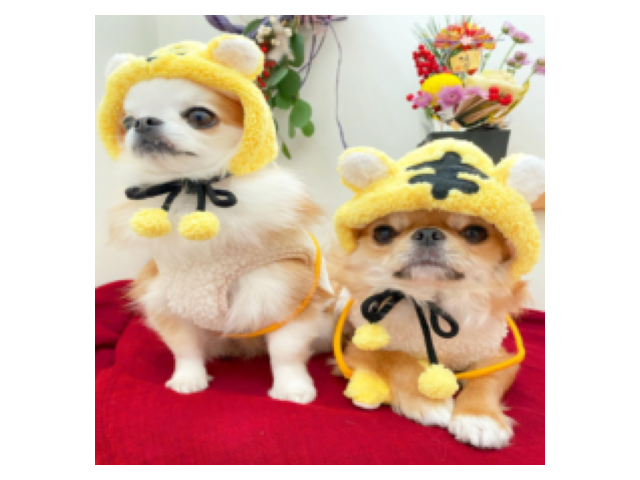}\hspace{-0.5cm}
    \includegraphics[width=0.18\linewidth]{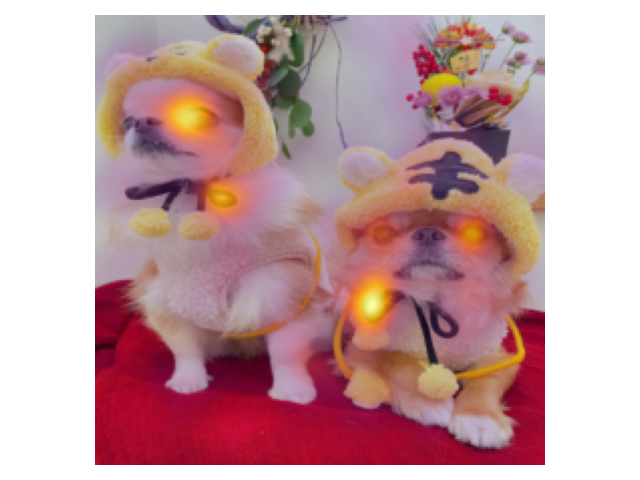}\hspace{-0.3cm}
    \includegraphics[width=0.18\linewidth]{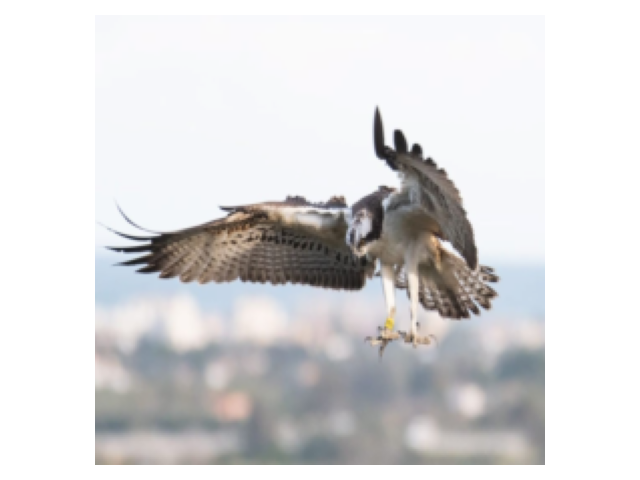}\hspace{-0.5cm}
    \includegraphics[width=0.18\linewidth]{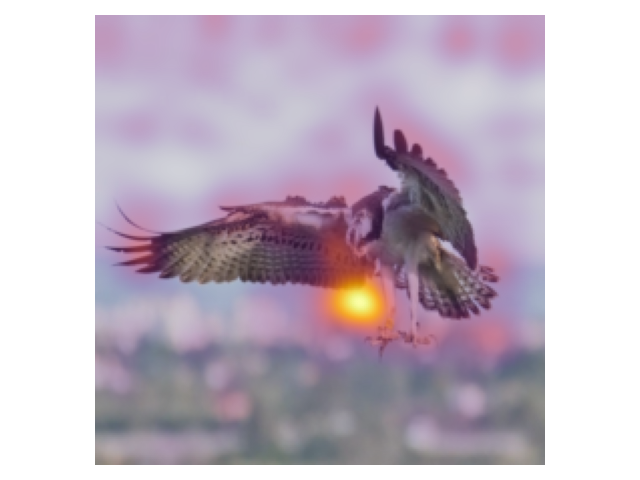}\hspace{-0.3cm}
    \includegraphics[width=0.18\linewidth]{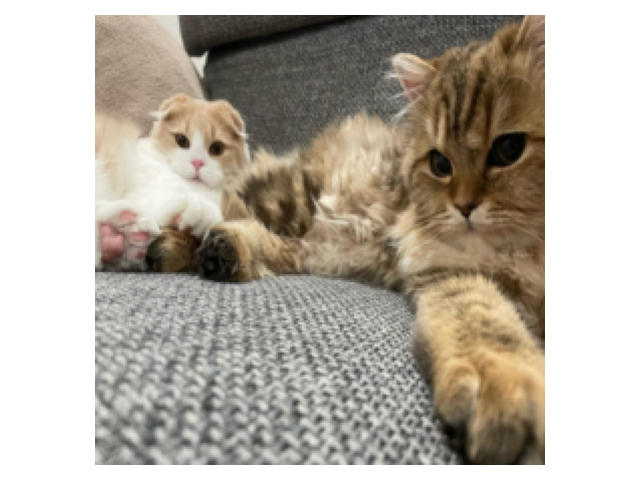}\hspace{-0.5cm}
    \includegraphics[width=0.18\linewidth]{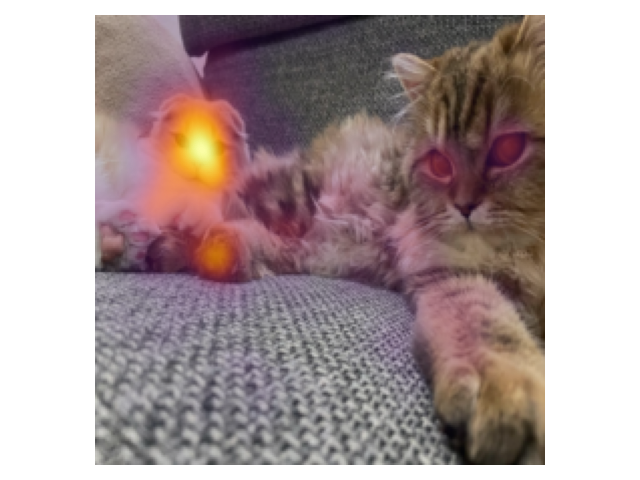}
    \caption{IG-OPTIMUS explanations.}
    \label{fig:IG}
\end{figure}

\begin{figure}[H]
    \centering
    \includegraphics[width=0.18\linewidth]{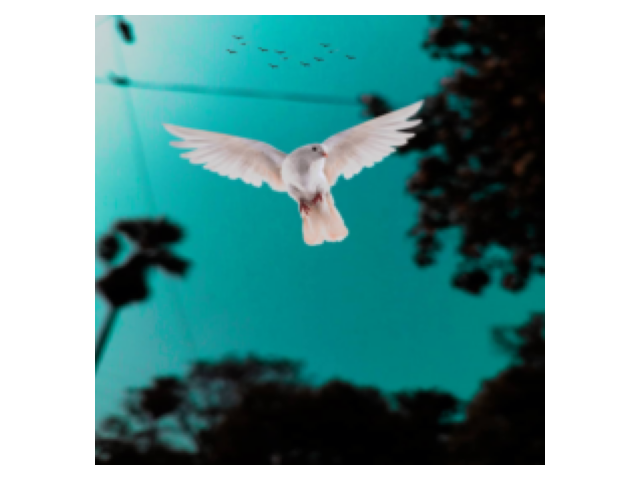}\hspace{-0.5cm}
    \includegraphics[width=0.18\linewidth]{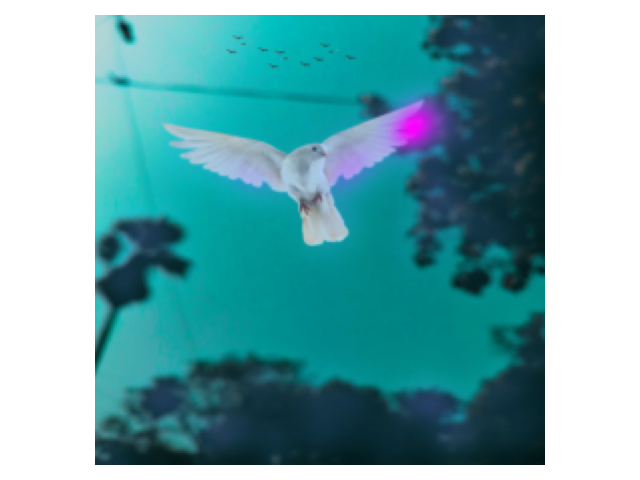}\hspace{-0.3cm}
    \includegraphics[width=0.18\linewidth]{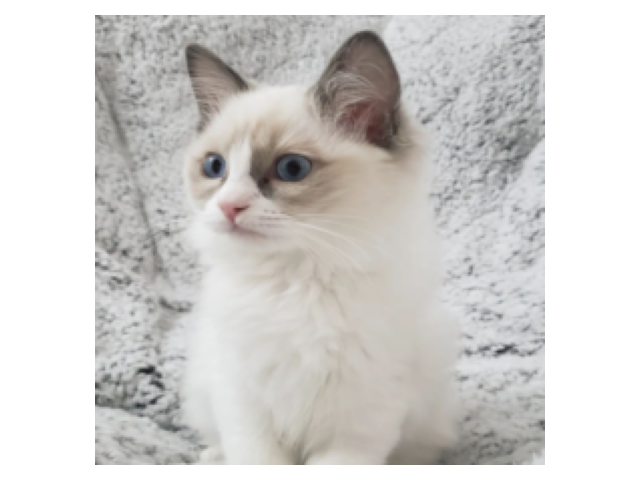}\hspace{-0.5cm}
    \includegraphics[width=0.18\linewidth]{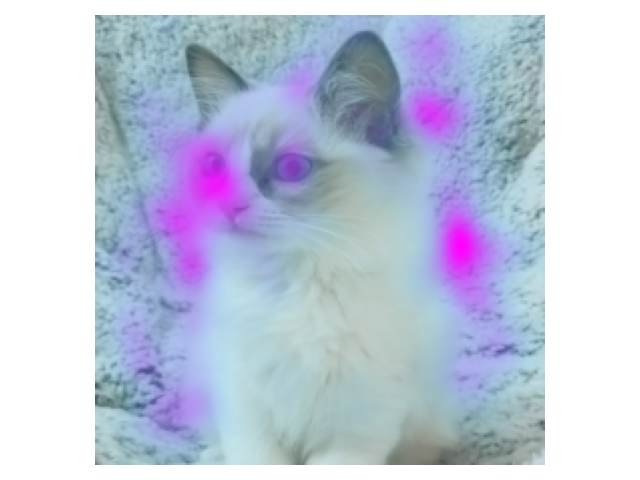}\hspace{-0.3cm}
    \includegraphics[width=0.18\linewidth]{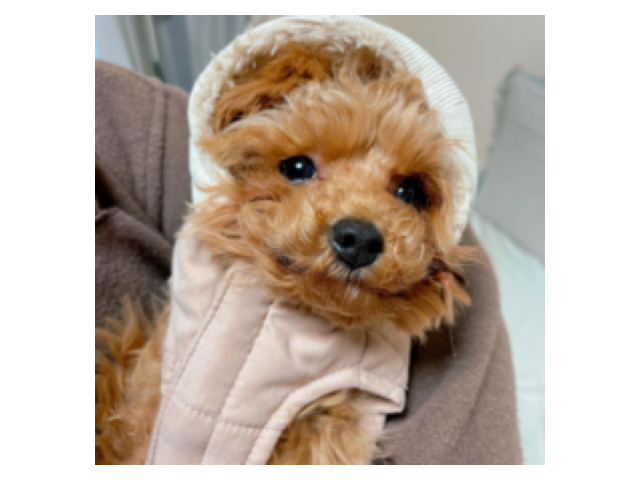}\hspace{-0.5cm}
    \includegraphics[width=0.18\linewidth]{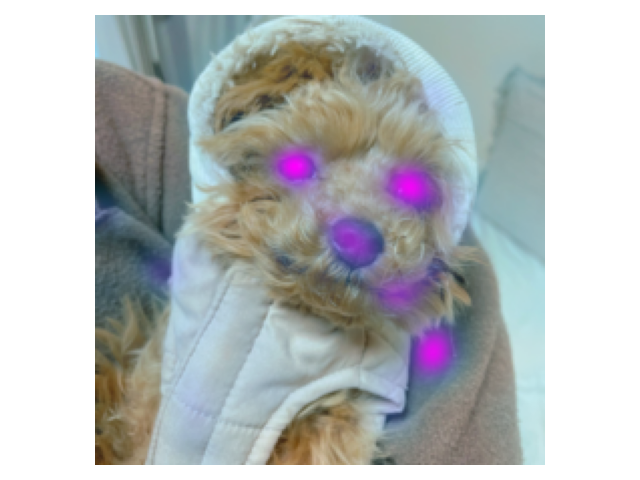}\\
    \includegraphics[width=0.18\linewidth]{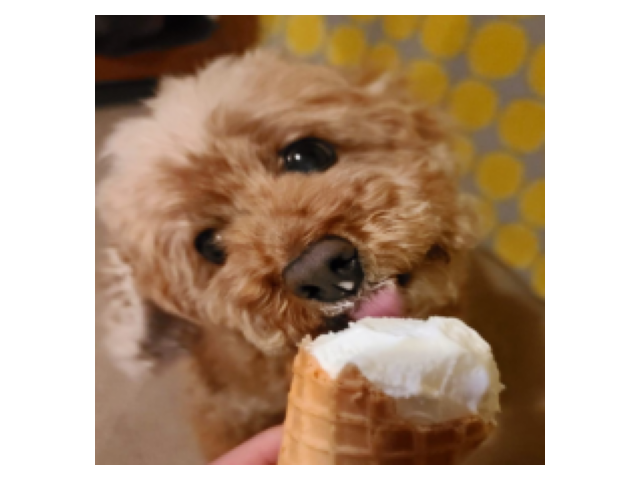}\hspace{-0.5cm}
    \includegraphics[width=0.18\linewidth]{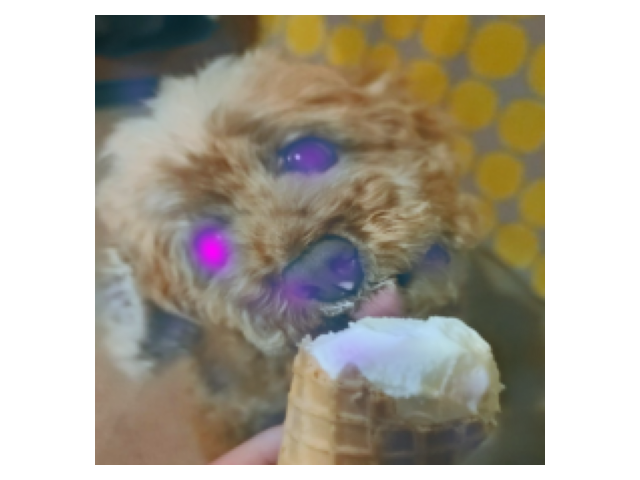}\hspace{-0.3cm}
    \includegraphics[width=0.18\linewidth]{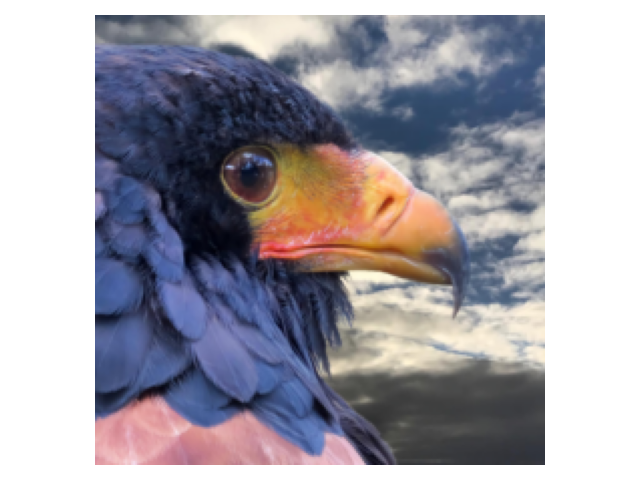}\hspace{-0.5cm}
    \includegraphics[width=0.18\linewidth]{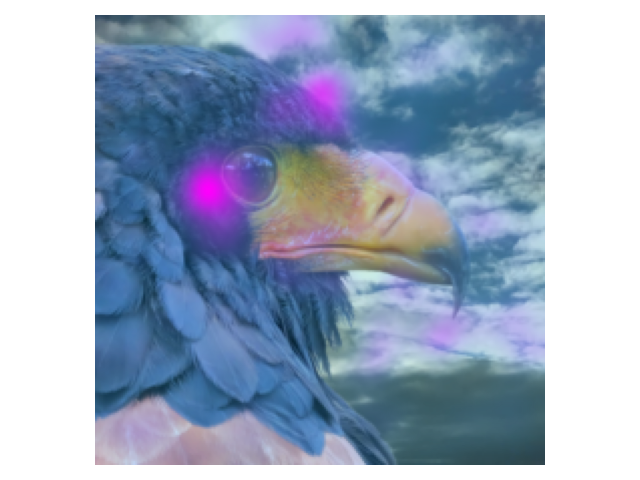}\hspace{-0.3cm}
    \includegraphics[width=0.18\linewidth]{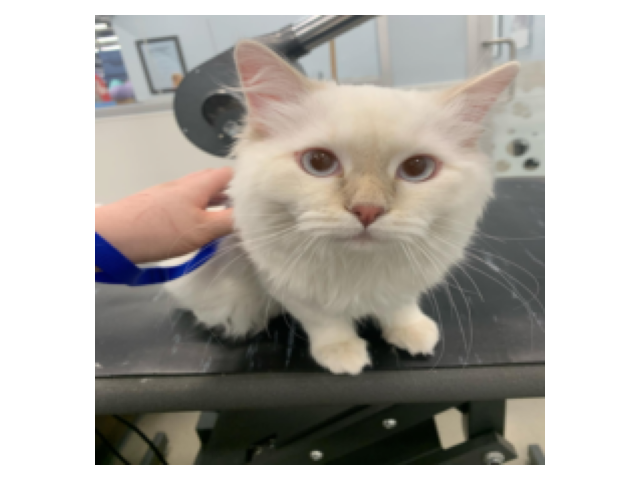}\hspace{-0.5cm}
    \includegraphics[width=0.18\linewidth]{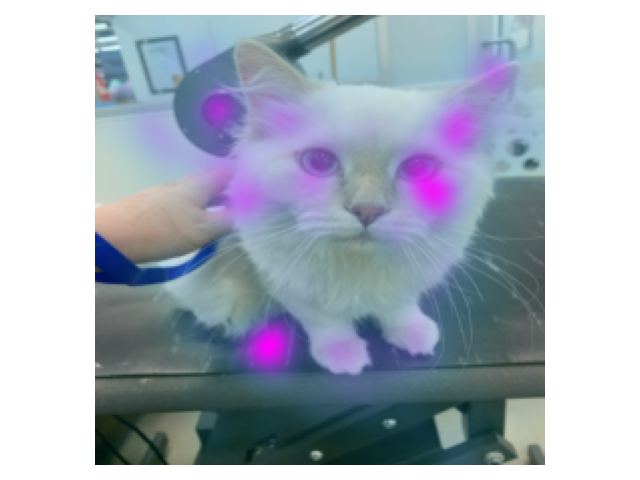}
    \caption{Difference between IG and IG-OPTIMUS: unnecessary concepts.}
    \label{fig:IG-diff}
\end{figure}

\end{document}